%% file: main.tex
\documentclass{article}

\usepackage{PRIMEarxiv}

\usepackage[utf8]{inputenc} 
\usepackage[T1]{fontenc}    
\usepackage{hyperref}       
\usepackage{url}            
\usepackage{booktabs}       
\usepackage{amsfonts}       
\usepackage{nicefrac}       
\usepackage{microtype}      
\usepackage{lipsum}
\usepackage{fancyhdr}       
\usepackage{graphicx}       
\graphicspath{{media/}}     
\usepackage{amsthm}
\usepackage{amsmath,amssymb,amsfonts,mathtools,bm}
\usepackage{xcolor}
\usepackage{algorithm}
\usepackage{algorithmic}
\usepackage{wrapfig}
\usepackage{subcaption}
\usepackage{hyperref}
\usepackage{cleveref}
\hypersetup{hidelinks}
\usepackage{xspace}
\usepackage{diagbox}
\usepackage{pgfplotstable}
\usepackage{svg}
\usepackage{multirow}

\usepackage{pgfplots}
\usepgfplotslibrary{colorbrewer}
\usepgfplotslibrary{groupplots} 
\usetikzlibrary{external}
\DeclareMathOperator*{\argmin}{arg\,min}
\newtheorem{theorem}{Theorem}
\newtheorem{lemma}{Lemma}
\newtheorem{claim}{Claim}

\newtheorem{assumption}{Assumption}

\newenvironment{claimproof}{\noindent\emph{Proof of claim.}}{\hfill$\qed$}
\crefname{lemma}{lemma}{lemmas}
\crefname{assumption}{assumption}{assumptions}
\newcommand{\abs}[1]{\lvert#1\rvert}
\newcommand{\vect}[1]{\ensuremath{\bm{#1}}}
\newcommand{\E}{\ensuremath{\mathbb{E}}}

\newcommand{\norm}[1]{\left\lVert#1\right\rVert}
\newcommand{\parenthese}[1]{\left(#1\right)}
\newcommand{\bracket}[1]{\left[#1\right]}

\newcommand{\scalarproduct}[1]{\left\langle#1\right\rangle}
\newcommand{\curlybracket}[1]{\left\{#1\right\}}

\newcommand{\ocal}{\ensuremath{\mathcal{O}}}
\newcommand{\dcal}{\ensuremath{\mathcal{D}}}
\newcommand{\rcal}{\ensuremath{\mathcal{R}}}
\newcommand{\fcal}{\ensuremath{\mathcal{F}}}
\newcommand{\kcal}{\ensuremath{\mathcal{K}}}
\newcommand{\indi}{\ensuremath{\mathbb{I}}}

\newcommand{\opt}[1]{\ensuremath{\bm{#1}^*}}

\newcommand{\ce}{DeLMFW\xspace}

\newcommand{\oco}{OCO\xspace}
\newcommand{\fw}{FW\xspace}
\newcommand{\ftpl}{FTPL\xspace}

\newboolean{showcomments}
\setboolean{showcomments}{true}
\ifthenelse{\boolean{showcomments}}
{ \newcommand{\mynote}[3]{
		\fbox{\bfseries\sffamily\scriptsize#1}
		{\small$\blacktriangleright$\textsf{\emph{\color{#3}{#2}}}$\blacktriangleleft$}}
	\newcommand{\zzz}[1]{{\setlength{\fboxsep}{2pt}\fcolorbox{black}{yellow}{\textsf{\emph{#1}}}}\xspace}}
{ \newcommand{\mynote}[3]{}
	\newcommand{\zzz}[1]{}}

\title{Handling Delayed Feedback in Distributed Online Optimization : A Projection-Free Approach
\thanks{The authors express their gratitude for the support provided by the Multidisciplinary Institute in Artificial Intelligence, University Grenoble-Alpes, France (ANR-19-P3IA-0003)} 
}
\author{
  Tuan-Anh Nguyen \\
  LIG, Inria \\ 
  University Grenoble-Alpes \\
  Saint-Martin-d'Hères, France \\
  \texttt{tuan-anh.nguyen@inria.fr}
  \And 
  Nguyen Kim Thang \\
  LIG, Inria, CNRS, Grenoble INP \\
  University Grenoble-Alpes \\
  Saint-Martin-d'Hères, France\\
  \texttt{kim-thang.nguyen@univ-grenoble-alpes.fr} \\
   \And
  Denis Trystram \\
  LIG, Inria, CNRS, Grenoble INP \\
  University Grenoble-Alpes \\
  Saint-Martin-d'Hères, France \\
  \texttt{denis.trystram@imag.fr} \\
}

\begin{document}
\maketitle

\begin{abstract}
Learning at the edges has become increasingly important as large quantities of data are continually generated locally. Among others, this paradigm requires 
algorithms that are \emph{simple} (so that they can be executed by local devices), \emph{robust} (again uncertainty as data are continually generated), and 
\emph{reliable} in a distributed manner under network issues, especially delays.
In this study, we investigate the problem of online convex optimization (\oco) under adversarial delayed feedback. We propose two projection-free algorithms for centralized and distributed settings in which they are carefully designed to achieve a regret bound of $\ocal(\sqrt{B})$ where $B$ is the sum of delay, which is optimal for the OCO problem in the delay setting while still being projection-free. We provide an extensive theoretical study and experimentally validate the performance of our algorithms by comparing them with existing ones on real-world problems. 

\end{abstract}

\keywords{Online Learning \and Distributed Learning \and Delayed Feedback \and Frank-Wolfe Algorithm}

\input{chapters/introduction}
\input{chapters/main-result}
\input{chapters/experiments}

\input{chapters/conclusion}

\bibliographystyle{plain}  
\bibliography{references} 

\newpage
\begin{center}
\rule{\textwidth}{.5pt} \\[3pt]
\textbf{\large Supplementary Material} \\[7pt]
\textbf{\large Handling Delayed Feedback in Distributed Online Optimization : A Projection-Free Approach}
\rule{\textwidth}{.5pt}
\end{center}
\appendix
\input{chapters/appendix.tex}

\input{chapters/mfw-proofs}
\input{chapters/decentralized-proofs2.tex}
\end{document}

%% file: chapters/introduction.tex

\section{Introduction}

Many machine learning (ML) applications owe their success to factors such as efficient optimization methods, effective system design, robust computation, and the availability of enormous amounts of data. In a typical situation, ML models are trained in an offline and centralized manner. However, in real-life scenarios, significant portions of data are continuously generated locally at the user level. Learning at the edge naturally emerges as a new paradigm to address such issues. In this new paradigm, the development of suitable learning techniques has become a crucial research objective. Responding to the requirements (of this new paradigm), online learning has been intensively studied in recent years. Its efficient use of computational resources, adaptability to changing environments, scalability, and robustness against uncertainty show promise as an effective approach for edge devices.

However, online learning/online convex optimization (OCO) problems typically assume that the feedback is immediately received after a decision is made, which is too restrictive in many real-world scenarios. For example, a common problem in online advertising is the delay that occurs between clicking on an ad and taking subsequent action, such as buying or selling a product. In distributed systems, the previous assumption is clearly a real issue. Wireless sensor/mobile networks that exchange information sequentially may experience delays in feedback due to several problems: connectivity reliability, varying processing/computation times, heterogeneous data and infrastructures, and unaware-random events. This can lead to difficulties in maintaining coordination and efficient data exchange, eventually affecting network performance and responsiveness. Given these scenarios, the straightforward application of traditional OCO algorithms often results in inefficient resource utilization because one must wait for feedback before starting another round. To address this need, this paper focuses on developing algorithms that can adapt to adversarial delayed feedback in both centralized and distributed settings.

\paragraph{Model.} 
We first describe the delay model in a centralized setting. Given a convex set $\mathcal{K} \subseteq \mathbb{R}^d$, at every time step $t$, the decision maker/agent chooses a decision $\vect{x}_t \in \mathcal{K}$ and suffers from a loss function $f_t : \mathcal{K} \rightarrow \mathbb{R}$. We denote by $d_t \geq 1$ an arbitrary delay value of time $t$. In contrast to the classical OCO problem, the feedback of iteration $t$ is revealed at time $t + d_t - 1$. The agent does not know $d_{t}$ in advance and is only aware of 
the feedback of iteration $t$ at time $t + d_t - 1$. Consequently, at time $t$, the agent receives feedback from the previous iterations $s \in \mathcal{F}_t$, where $\mathcal{F}_t = \curlybracket{s: s + d_s -1 = t}$. In other words, $\mathcal{F}_{t}$ is the set of moments before time $t$ such that the corresponding feedbacks are released at time $t$. Moreover, the corresponding feedbacks are not necessarily released in the order of their iterations.  The goal is to minimize regret, which is defined as:
\begin{align*}
    \mathcal{R}_T := \sum_{t = 1}^T f_t(\vect{x}_t) - \min_{\vect{x} \in \mathcal{K}} \sum_{t=1}^T f_t(\vect{x}) 
\end{align*}
In a distributed setting, we have additionally a set of agents connected over a network, 
represented by a graph $\mathcal{G} = (V, E)$  where $n  = |V|$ is the number of agents. 
Each agent $i \in V$ can communicate with (and only with) its immediate neighbors, that is, adjacent agents in $\mathcal{G}$.
At each time $t \geq 1$, agent $i$ takes a decision $\vect{x}^{i}_{t} \in \mathcal{K}$ and suffers a partial loss function $f^i_t : \mathcal{K} \rightarrow \mathbb{R}$, which is revealed adversarially and locally to the agent at time $(t + d^i_t - 1)$ --- again, that is unknown to the agent. 
Similarly, denote $\mathcal{F}^i_t = \{s: s + d^i_s - 1 = t\}$ as the set of feedbacks revealed to agent $i$ at time 
$t$ where $d^i_s$ is the delay of iteration $s$ to agent $i$.
Although the limitation in communication and information, the agent $i$ is interested in the global loss $F_t(.)$ where $F_t (.) = \frac{1}{n} \sum_{i=1}^n f^i_t (.)$. In particular, at time $t$, the loss of agent $i$ for chosen $\vect{x}^i_t$ is $F_t (\vect{x}^i_t)$. 
Note that each agent $i$ does not know $F_{t}$ but has only knowledge of $f^i_{t}$ --- its observed cost function. 
The objective here is to minimize regret for all agents: 
$$
\mathcal{R}_T := \max_{i} \biggl( \sum_{t = 1}^T F_t(\vect{x}^i_t) - \min_{\vect{x} \in \mathcal{K}} \sum_{t=1}^T F_t(\vect{x}) \biggr)
$$
\begin{wrapfigure}{r}{0.5\textwidth}
    \vspace{+20pt}
    \includegraphics[width=0.5\textwidth]{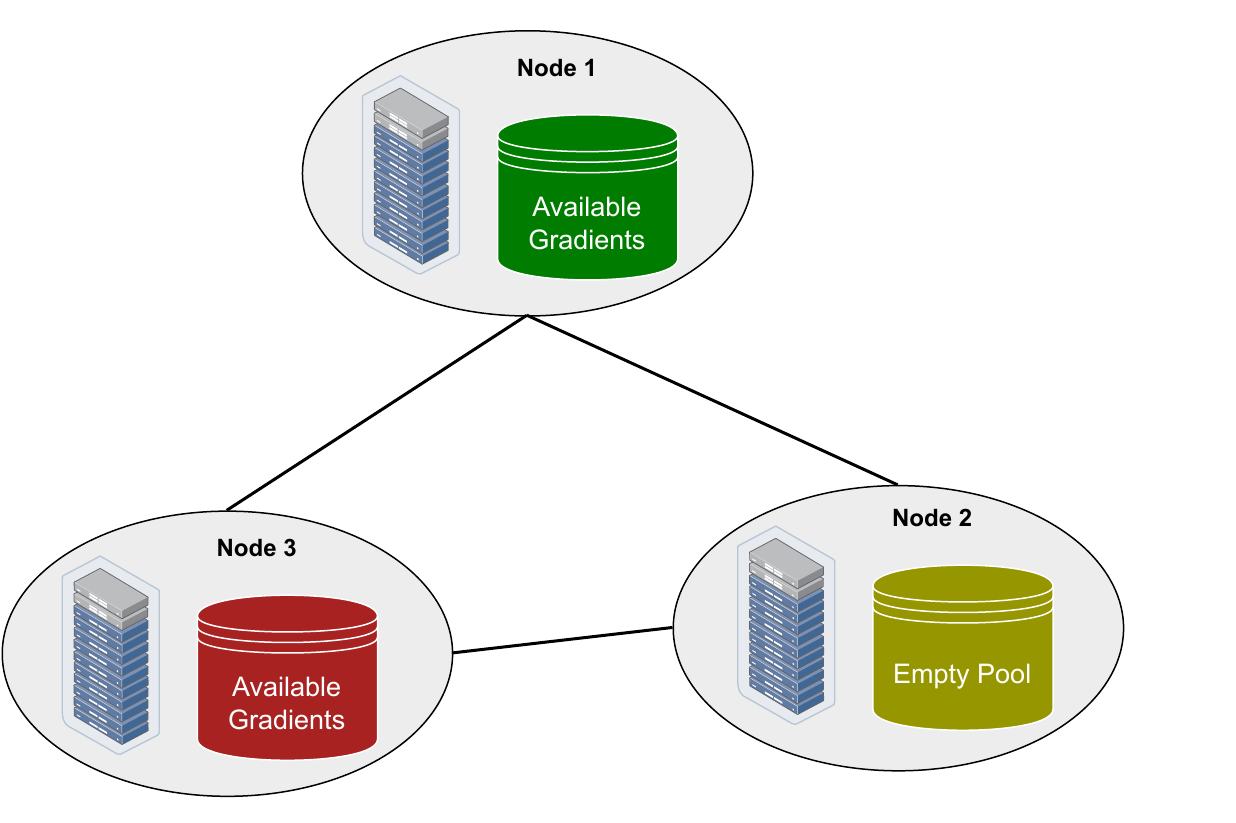}
    \caption{Illustration of delayed feedback in distributed system. Given a time $t$, each agent holds a distinct pool of available gradient feedback from $s <t$ that is ready for computation at the current time. The pool can also be empty if no feedback is provided.}
    \vspace{-60pt}
\end{wrapfigure}

\subsection{Our contribution}
The challenge in designing robust and efficient algorithms for these problems is to address the following issues simultaneously:
\begin{itemize}
    \item Uncertainty (online setting, agents observe their loss functions only after selecting their decisions).
    \item Asynchronous (distributed setting with different delayed feedback between agents)
    \item Partial information (distributed setting, agents know only its local loss functions while attempting to minimize the cumulative loss).
    \item Low computation/communication resources of agents (so it is desirable that each agent performs a small number of gradient computations and communications). 
\end{itemize}

We introduce performance-guaranteed algorithms in solving the centralized and distributed constraint online convex optimization problem with adversarial delayed feedback. Our algorithms achieve an \emph{optimal} regret bound for centralized and distributed settings.
Specifically, we obtain the regret bound of $O(\sqrt{B})$ where $B$ is the total delay in the centralized setting and 
$B$ is the average total delay over all agents in the distributed setting. Note that, if $d$ is a maximum delay of each feedback then 
our regret bound becomes $\mathcal{O}(\sqrt{dT})$. This result recovers the regret bound of $O(\sqrt{T})$ in the classic setting without delay (i.e., $d = 1$). Additionally, the algorithms can be made projection-free by selecting appropriate oracles, allowing them to be implemented in different contexts based on the computational capacity of local devices. Finally, we illustrate the practical potential of our algorithms and provide a thorough analysis of their performance which is predictably explained by our theoretical results. The experiments demonstrate that our proposed algorithms outperform existing solutions in both synthetic and real-world datasets.

\begin{table*}[hbt!]
\caption{Comparisons to previous algorithms DGD \cite{quanrud} and DOFW \cite{OFW-delay} on centralized online convex optimization with delays bounded by $d$. Our algorithms are in bold.}
\centering
\begin{tabular*}{\linewidth}{c @{\extracolsep{\fill}} ccccc}
\toprule
Algorithm & Centralized & Distributed & Adversarial Delay & Projection-free & Regret \\
\midrule
DGD & Yes & - & Yes  & - & $\mathcal{O} (\sqrt{dT})$ \\ 
DOFW  & Yes & - & Yes & Yes &$\mathcal{O}(T^{3/4} + dT^{1/4})$ \\
\textbf{DeLMFW}  & Yes & - & Yes & Yes &$\mathcal{O}(\sqrt{dT})$ \\
\textbf{De2MFW} & - & Yes & Yes & Yes &$\mathcal{O}(\sqrt{dT})$ \\
\bottomrule
\end{tabular*}
\hspace{0.8cm}
\label{tab:dataset}
\end{table*}
\input{chapters/related}

%% file: chapters/related.tex
\subsection{Related Work}

\paragraph{Online Optimization with delayed feedback}

Over the years, studies on online optimization with delayed feedback have undergone a swift evolution. \cite{learner-are-fast} shed light on the field by focusing on the convergence properties of online stochastic gradient descent with delays. They provide a regret bound of $\mathcal{O}(\sqrt{d T})$ with $d$ the delay value if $d^2 \leq T$. Later on, \cite{quanrud} proposes a centralized (single-agent) gradient descent algorithm under adversarial delays. The theoretical analysis of \cite{quanrud} entails a regret bound of $\mathcal{O}(\sqrt{B})$, where $B$ is the total delay. This bound becomes $\mathcal{O}(\sqrt{d T})$ if $d$ is the upper bound of delays. \cite{bold-ofw} provided a black-box style method to learn under delayed feedback. They showed that for any non-delayed online algorithms, the additional regret in the presence of delayed feedback depends on its prediction drifts. \cite{constrained-oco-delay} developed an online saddle point algorithm for convex optimization with feedback delays. They achieved a sublinear regret $\mathcal{O}(\sqrt{dT})$ where $d$ is a fixed constant delay value. Recently, \cite{OFW-delay} proposed a first Frank-Wolfe-type online algorithm with delayed feedback. They modified the Online Frank-Wolfe (OFW) for the unknown delays setting and provided a regret bound of $\mathcal{O}(T^{3/4} + dT^{1/4})$. 
This is the current state of the art for projection-free (Frank-Wolfe-type) algorithms with delays. 
Our bound of $O(\sqrt{dT})$ improves over the aforementioned results.

\paragraph{Distributed Online Optimization.}  %
\cite{Yan:2013} introduced decentralized online projected subgradient descent and showed vanishing regret for convex and strongly convex functions.  In contrast,\cite{Hosseini:2013} extended distributed dual averaging technique to the online setting, using a general regularized projection for both unconstrained and constrained optimization.
A distributed variant of online conditional gradient~\cite{Hazanothers16:Introduction-to-online} was designed and analyzed in~\cite{Zhang:2017} that requires linear minimizers and uses exact gradients. Computing exact gradients may be prohibitively expensive for moderately sized data and intractable when a closed form does not exist. \cite{THANG2022334} proposes a decentralized online algorithm for convex function using stochastic gradient estimate and multiple optimization oracles. This work achieves the optimal regret bound of $O(T^{1/2})$ and requires multiple gradient evaluation and communication rounds. Later on, \cite{nguyen23a} provide a decentralized algorithm that uses stochastic gradient estimate and reduces communication by using only one gradient evaluation. \cite{survey-distributed} provides a comprehensible survey on recent developpement of distributed \oco.
More recent work on distributed online optimization with feedback delays is proposed in \cite{decentralized-feedback-delays}. The authors consider a distributed projected gradient descent algorithm where each agent has a fixed known amount of delay $d_i$. They provide a regret bound of $\mathcal{O} (\sqrt{d T})$ where $d = \max_{i} d_{i}$ but the delays $d_{i}$ must be fixed (non-adversarial). 

Despite the growing number of studies on decentralized online learning in recent years, there is a lack of research that accounts for the \emph{adversarial/online} delayed feedback. In this paper, we first present a centralized online algorithm and then extend it to a distributed online variant that takes an adversarial delay setting into consideration. 

%% file: chapters/main-result.tex
\section{Projection-Free Algorithms under Delayed Feedback}
In this section, we will present our method for addressing delayed feedback in the \oco problem. In \Cref{sec:preliminaries}, we state some assumptions and results that form the basis of our approach.  We then describe our first algorithm \ce for centralized setting in \Cref{sec:centralized} and extend it to distributed setting in \Cref{sec:distributed}.
\input{chapters/notations.tex}
\input{chapters/mfw}
\input{chapters/decentralized-dmfw}

%% file: chapters/notations.tex
\subsection{Preliminaries}
\label{sec:preliminaries}
Throughout the paper, we use boldface letter e.g $\vect{x}$ to represent vectors. We denote by $\vect{x}_t$ the final decision of round $t$ and $\vect{x}_{t,k}$ to be the sub-iterate at round $k$ of $t$. In distributed setting, we add a superscript $i$ to make distinction between agents. If not specified otherwise, we use Euclidean norm $\norm{.}$ and suppose that the constraint set $\mathcal{K} \subset \mathbb{R}^m$ is convex. We state the following standard assumptions in \oco.

\begin{assumption}
\label{assum:boundedness}
The constraint set $\mathcal{K}$ is a bounded convex set with diameter $D$ i.e $D := \sup_{\vect{x}, \vect{y} \in \mathcal{K}} \norm{\vect{x} - \vect{y}}$
\end{assumption}

\begin{assumption}[Lipschitz]
\label{assum:lipschitz}
For all $\vect{x} \in \kcal$, there exists a constant $G$ such that, $\forall t \in [T]$, $\norm{\nabla f_t (\vect{x})} \leq G$
\end{assumption}

\begin{assumption}[Smoothness]
\label{assum:smooth}
For all $\vect{x}, \vect{y} \in \kcal$, there exists a constant $\beta$ such that, $\forall t \in [T]$ :
\begin{align*}
    f_t(\vect{y}) \leq f_t(\vect{x}) + \scalarproduct{\nabla f_t(\vect{x}), \vect{y}-\vect{x}} + \frac{\beta}{2} \norm{\vect{y}-\vect{x}}^2 
\end{align*}
or equivalently $\norm{\nabla f_t(\vect{x}) - \nabla f_t(\vect{y})} \leq \beta \norm{\vect{x} - \vect{y}}$.
\end{assumption}

\paragraph{Online Linear Optimization Oracles} 
In the context of the Frank-Wolfe (\fw) algorithm, we utilize multiple optimization oracles to approximate the gradient of the upcoming loss function by solving an online linear problem. This approach was first introduced in \cite{ChenHarshaw18:Projection-Free-Online}. Specifically, the online linear problem involves selecting a decision $\vect{v}_t \in \kcal$ at every time $t \in [T]$. The adversary then reveals a vector $\vect{g}_t$ and loss function $\scalarproduct{\vect{g}_t, \cdot}$ to the oracle. The objective is to minimize the oracle's regret. A possible candidate for an online linear oracle is the Follow the Perturbed Leader algorithm (\ftpl) \cite{KALAI2005291}. Given a sequence of historical loss functions $\scalarproduct{\vect{g}_{\ell}, \cdot}, s \in [1, t]$ and a random vector $\vect{n}$ drawn uniformly from a probability distribution \dcal, \ftpl makes the following update.
\begin{align}
\label{eq:ftpl-update}
    \hat{\vect{v}}_{t+1} = \argmin_{\vect{v} \in \kcal} \curlybracket{
        \zeta \sum_{\ell = 1}^{t} \scalarproduct{\vect{g}_{\ell} , \vect{v}} 
            + \scalarproduct{\vect{n}, \vect{v}}
    }
\end{align}

\begin{lemma}[Theorem 5.8 \cite{Hazanothers16:Introduction-to-online}]
\label{lmm:ftpl-regret}
Given a sequence of linear loss function $f_1, \ldots, f_T$. Suppose that \Cref{assum:boundedness,assum:lipschitz,assum:smooth} hold true. Let $\dcal$ be a
the uniform distribution over hypercube $\bracket{0,1}^m$. The regret of FTPL is 
\begin{align*}
    \rcal_{T, \ocal} \leq \zeta DG^2 T + \frac{1}{\zeta} \sqrt{m} D
\end{align*}
where $\zeta$ is learning rate of algorithm. 
\end{lemma}

\paragraph{Delay Mechanism}
We consider the following delay mechanism. At round $t$, the agent receives a set of delayed gradient $\nabla f_s (\vect{x}_s)$ from previous rounds $s \leq t$ such that $s + d_s -1 = t$, where $d_s$ is the delay value of iteration $s$. We denote by $\fcal_t = \curlybracket{s:s + d_s - 1 = t}$ the set of indices released at round $t$. Following this setting, the feedback of round $t$ is released at time $t + d_t -1$, and the case $d_t = 1$ is considered as no delay. We suppose that the delay value is unknown to the agent and make no assumption about the set $\fcal_t$. Consequently it is possible for the set to be empty at any particular round. We extend the aforementioned mechanism to the distributed setting by assuming that each agent has a unique delay value at each round $t \in [T]$. The delay value of agent $i$ at round $t$ is denoted by $d^i_t$, and the set of delayed feedbacks of agent $i$ at round $t$ is denoted by $\fcal^i_t = \curlybracket{s:s + d^i_s - 1 = t}$, which is distinct between agents.

%


%% file: chapters/mfw.tex
\subsection{Centralized Algorithm}
\label{sec:centralized}

We describe the procedure of \Cref{algo:delmfw} in details. At each round $t$, the agent performs two blocks of operations: prediction and update. During the prediction block, the agent performs $K$ iterations of \fw updates by querying solutions from the oracles $\ocal_k$, $k \in [K]$ and updates the sub-iterate vector $\vect{x}_{t,k+1}$ using a convex combination of the previous one and the oracle's output. The agent then plays the final decision $\vect{x}_t = \vect{x}_{t,K+1}$ and incurs a loss $f_t (\vect{x}_t)$ which may not be revealed at $t$ due to delay. From the mechanism described in \Cref{sec:preliminaries}, there exists a set of gradient feedbacks from the previous rounds revealed at $t$ whose indices are in $\fcal_t$. The update block involves observing the delayed gradients evaluated at $K$ sub-iterates of rounds $s \in \fcal_t$, computing surrogate gradients $\curlybracket{\vect{g}_{t,k}, k \in [K]}$ by summing the delayed gradients and feeding them back to the oracles $\curlybracket{\ocal_k, k \in K}$.

In our algorithm, the agent employs a suite of online linear optimization oracles, denoted $\ocal_1, \ldots, \ocal_K$. These oracles utilize feedbacks accumulated from previous rounds to estimate the gradient of the upcoming loss function. However, in the delay setting, these estimations may be perturbed owing to a lack of information. For example, if there is no feedback from rounds $t$ to $t'$, that is, $\fcal_s = \emptyset$ for $s \in [t, t']$, the oracles will resort to the information available in round $t-1$ to estimate the gradient of all rounds from $t+1$ to $t'+1$. As a result, the oracle's output remains unchanged for these rounds, and decisions $\curlybracket{\vect{x}_s: s \in [t+1, t'+1]}$ are not improved. Our analysis for \Cref{algo:delmfw,algo:de2mfw} will be focused on assessing the impact of delayed feedback on the oracle's output.

\begin{lemma}
\label{lmm:ftpl-cost-delay}
Let $\hat{\vect{v}}_t$ be the \ftpl prediction defined in \Cref{eq:ftpl-update} and 
\begin{align*}
    \vect{v}_t = \argmin_{\vect{v} \in \kcal} \curlybracket{\zeta \sum_{\ell = 1}^{t-1} \scalarproduct{\sum_{s \in \fcal_\ell}\vect{g}_s , \vect{v}} + \scalarproduct{\vect{n}, \vect{v}}}
\end{align*}
the prediction of \ftpl with delayed feedback.
For all $t \in [T]$, we have: 
\begin{align*}
    \norm{\vect{v}_t - \hat{\vect{v}}_t} \leq \zeta DG \sum_{s < t} \indi_{\curlybracket{s + d_s > t}}
\end{align*}
\end{lemma}

\begin{algorithm}
\algsetup{linenosize=\tiny}
\small
\begin{flushleft}
\textbf{Input}:  Constraint set $\mathcal{K}$, 
	number of iterations $T$, sub-iteration $K$, online oracles $\curlybracket{\mathcal{O}_{k}}_{k=1}^K$, step sizes $\eta_k \in (0, 1]$
\end{flushleft}
\begin{algorithmic}[1]
\FOR {$t = 1$ to $T$}	 		
    \STATE Initialize arbitrarily $\vect{x}_{t,1} \in \mathcal{K}$ 
    \FOR{$k=1$ to $K$}
        \STATE Query $\vect{v}_{t,k}$ from oracle $\mathcal{O}_{k}$.
        \STATE $\vect{x}_{t,k+1} \gets (1 - \eta_{k}) \vect{x}_{t,k} + \eta_{k} \vect{v}_{t,k}$.
    \ENDFOR
    \STATE $\vect{x}_{t} \gets \vect{x}_{t,K+1}$, play $\vect{x}_{t}$ and incurs loss $f_t(\vect{x}_t)$ 
    \STATE Receive $\mathcal{F}_t = \curlybracket{s \in [T]: s + d_s -1 = t} $  
    \IF {$\mathcal{F}_t = \emptyset$} 
        \STATE do nothing 
    \ELSE
        \FOR{$k=1$ to $K$}
            \STATE $\vect{g}_{t,k} \gets \sum_{s \in \mathcal{F}_t} \nabla f_s (\vect{x}_{s,k}) $
            \STATE Feedback $\langle \vect{g}_{t,k}, \cdot \rangle$ 
    				to oracles $\mathcal{O}_{k}$. 
        \ENDFOR            
    \ENDIF
\ENDFOR
\end{algorithmic}
\caption{DeLMFW}
\label[algorithm]{algo:delmfw}
\end{algorithm}

\begin{theorem}
\label[theorem]{thm:centralized}
Given a constraint set $\mathcal{K}$. Let $A = \max\curlybracket{3, \frac{G}{\beta D}}$, $\eta_k = \min \curlybracket{1,\frac{A}{k}}$, and $K = \sqrt{T}$. Suppose that \Cref{assum:boundedness,assum:lipschitz,assum:smooth} hold true. If we choose FTPL as the underlying oracle and set $\zeta = \frac{1}{G \sqrt{B}}$, the regret of \Cref{algo:delmfw} is 
\begin{align}
    \sum_{t=1}^T &\bracket{f_t(\vect{x}_t) - f_t(\vect{x}^*)} 
    \leq 2 \beta A D^2 \sqrt{T}  
        + 3(A+1)\parenthese{DG\sqrt{B} + \rcal_{T, \ocal}}
\end{align}
where $B = \sum_{t=1}^T d_t$, the sum of all delay values and $\rcal_{T,\ocal}$ is the regret of FTPL with respect to the current choice of $\zeta$.
\end{theorem}
\paragraph{Discussion}
The regret bound of \Cref{thm:centralized} differs from that of the non-delayed MFW \cite{ChenHarshaw18:Projection-Free-Online} by the additive term $DG\sqrt{B}$ which represents the total cost of sending delayed feedback to the oracles over $T$ rounds (\Cref{lmm:ftpl-cost-delay}). If we assume that there exists a maximum value $d$ such that $d_t \leq d$ for all $t \in [T]$. Our regret bound becomes $\ocal (\sqrt{dT})$ which coincides with the setting in \cite{OFW-delay}, a delayed-feedback \fw algorithm that achieves $\ocal (T^{3/4} + dT^{1/4})$. Another line of work is from \cite{bold-ofw}, a framework that addresses delayed feedback for any base algorithm. By considering MFW as the base algorithm, their theoretical analysis suggests that the algorithm also achieves $\ocal(\sqrt{dT})$ regret bound. However, their delay value is not completely unknown to the agent because it is time-stamped by maintaining multiple copies of the base algorithm. We empirically show in \Cref{sec:exp} that this algorithm is highly susceptible to high delay values. Instead of using \ftpl, our algorithm has the flexibility to select any online algorithm as an oracle, for example, Online Gradient Descent \cite{Hazanothers16:Introduction-to-online}.

%% file: chapters/decentralized-dmfw.tex
\subsection{Distributed Algorithm}
\label{sec:distributed}
In this section, we extend \Cref{algo:delmfw} to a distributed setting in which multiple agents collaboratively optimize a global model. Our setting considers a fully distributed framework, characterized by the absence of a server to coordinate the learning process. Let $\mathbf{W} \in \mathbb{R}_{+}^{n \times n}$ be the adjacency matrix of communication graph $\mathcal{G} =(V,E)$. The entries $w_{ij}$ are defined as 
\begin{align*}
    w_{ij} = \begin{cases}
            \dfrac{1}{1+\max\{\tau_i, \tau_j\}} & \text{if $(i,j) \in E$}\\
            0 &  \text{if $(i,j) \not\in E$,$i \neq j$}\\
            1 - \sum_{j \in N(i)} w_{ij} & \text{if $i=j$}
        \end{cases}
\end{align*}
where $\tau_i = \# \{j \in V: (i,j) \in E\}$ is the degree of vertex $i$.
The matrix $\mathbf{W}$ is doubly stochastic (i.e $\mathbf{W} \vect{1} = \mathbf{W}^T \vect{1} = \vect{1}$) and therefore possesses several useful properties associated with doubly stochastic matrices. 
We call $\lambda (\mathbf{W})$ the second-largest eigenvalue of $\mathbf{W}$ and define $k_0$ as the smallest integer that verifies $\lambda(\mathbf{W}) \leq  \parenthese{\frac{k_0}{k_0 +1}}^2$. Furthermore, we set $\rho = 1 - \lambda (\mathbf{W})$ to be the spectral gap of matrix $\mathbf{W}$.

At a high level, each agent maintains $K$ copies of the oracles $\ocal^i_1, \cdots, \ocal^i_K$ while performing prediction and update at every round $t$. The prediction block consists of performing $K$ \fw-steps while incorporating the neighbors' information. Specifically, the agent computes at its local level during the $K$ steps a local average decision $\vect{y}^i_{t,k}$ representing a weighted aggregation of its neighbor's current sub-iterates. The update vector is convex combination of the local average decision and the oracle's output. The final decision of agent $\vect{x}^i_t$ is disclosed at the end of $K$ steps. \Cref{lmm:decision-distance} shows that $\vect{y}^i_{t,k}$ is a local estimation of the global average $\overline{\vect{x}}_{t,k} = \frac{1}{n} \sum_{i=1}^n \vect{x}^i_{t,k}$ as $K$ increases.

Following the $K$ \fw-steps, the update block employs $K$ gradient updates utilizing the delayed feedback from previous rounds. The agent observes the delayed gradients evaluated on theirs corresponding subiterates and computes the local average gradient $\vect{d}^i_{t,k} $ through a weighted aggregation of the neighbors' current surrogates (\ref{eq:gradient-avg}). The agent updates the surrogate gradient via a gradient-tracking step (\ref{eq:gradient-tracking1}) to ensure that it approaches the global gradient as $K$ increases.
It is worth noting that feedback provided to the oracle contains information about delays experienced by all neighboring agents. Consequently, the oracle $\ocal^i_k$ observes delayed feedback from $\cup_{j \in \mathcal{N}(i)} \fcal^j_{t}$ instead of $\fcal^i_t$. This result highlights the dependency on the connectivity of the communication graph when considering the effect of delayed feedback to the oracle's output, as demonstrated in \Cref{lmm:ftpl-bound-distributed}.
\begin{lemma}
\label{lmm:decision-distance}
	Define $C_d = k_0\sqrt{n} D$, for all $t \in [T]$, $k \in [K]$, we have
		\begin{align}
		\max_{i \in [1,n]} \norm{\vect{y}^{i}_{t,k} - \overline{\vect{x}}_{t,k}} \leq \frac{C_d}{k}
	\end{align}
\end{lemma}
\begin{lemma}
\label{lmm:ftpl-bound-distributed}
For all $t \in [T]$, $k \in [K]$ and $i \in [n]$. Let $\vect{v}^i_{t,k}$ be the output of the oracle $\ocal^i_k$ with delayed feedback and $\hat{\vect{v}}^i_{t,k}$ its homologous in non-delay case. Suppose that \Cref{assum:boundedness,assum:lipschitz} hold true. Choosing \ftpl as the oracle, we have:
	\begin{align}
		&\norm{\vect{v}^i_{t,k} - \hat{\vect{v}}^i_{t,k}} 
		\leq 2 \zeta \sqrt{n} DG \parenthese{\frac{\lambda\parenthese{\mathbf{W}}}{\rho} +1 } \frac{1}{n} \sum_{i=1}^n \sum_{s \leq t}  \indi_{\curlybracket{s + d^i_s > t}}
	\end{align}
where $\zeta$ is the learning rate, $\lambda(\mathbf{W})$ is the second-largest eigenvalue of $\mathbf{W}$ and $\rho = 1 - \lambda (\mathbf{W})$ is the spectral gap of matrix $\mathbf{W}$.
\end{lemma}
\begin{theorem}
    \label[theorem]{thm:decentralized}
    Given a constraint set $\mathcal{K}$. Let $A = \max\curlybracket{3,\frac{3G}{2\beta D}, \frac{2\beta C_d + C_g}{\beta D}}$, $\eta_k = \min \curlybracket{1,\frac{A}{k}}$, and $K = \sqrt{T}$. Suppose that \Cref{assum:boundedness,assum:lipschitz,assum:smooth} hold true. If we choose FTPL as the underlying oracle and set $\zeta=\frac{1}{G \sqrt{B}}$, the regret of \Cref{algo:de2mfw} is 
    \begin{align}
        &\sum_{t=1}^T  \bracket{F_t (\vect{x}^i_t) - F_t(\opt{x})}
        \leq \parenthese{GC_d + 2\beta A D^2} \sqrt{T} 
		 + 3(A+1)\parenthese{2\sqrt{n} DG \parenthese{\frac{\lambda\parenthese{\mathbf{W}}}{\rho} +1 } \sqrt{B}  + \rcal_{T, \ocal}} \nonumber
    \end{align}
    where $B = \frac{1}{n} \sum_{i=1}^n B_i$ such that $B_i$ is the sum of all delay values of agent $i$. $C_d = k_0 \sqrt{n} D$ and $C_g = \sqrt{n} \max \curlybracket{\lambda_2(\mathbf{W}) \parenthese{G + \frac{\beta D}{\rho}}, k_0\beta \parenthese{4C_d + AD}}$ and $\rcal_{T,\ocal}$ is the regret of FTPL with respect to the current choice of $\zeta$.
\end{theorem}

\begin{algorithm}[ht!]
	\algsetup{linenosize=\tiny}
	\small
	\begin{flushleft}
	\textbf{Input}:  Constraint set $\mathcal{K}$, 
		number of iterations $T$, sub-iterations $K$, online linear optimization oracles $\curlybracket{\mathcal{O}^i_k: k \in [K]}$ for each agent $i \in [n]$, step sizes $\eta_k \in (0, 1]$
	\end{flushleft}
	\begin{algorithmic}[1]
	\FOR {$t = 1$ to $T$}	 		
		\FOR{every agent $i = 1$ to $n$}	%
		\STATE Initialize arbitrarily $\vect{x}^i_{t,1} \in \mathcal{K}$ 
			\FOR{$k=1$ to $K$}
				\STATE Query $\vect{v}^i_{t,k}$ from oracle $\mathcal{O}^i_{k}$ 
				\STATE Exchange $\vect{x}^{i}_{t,k}$  with neighbours $\mathcal{N}(i)$
				\STATE $\vect{y}^{i}_{t,k} \gets \sum_{j} w_{ij} \vect{x}^j_{t,k}$
				\STATE \label{eq:convex-combi-dis} $\vect{x}^t_{i,k+1} \gets (1 - \eta_{k}) \vect{y}^i_{t,k} + \eta_{k} \vect{v}^i_{t,k}$
			\ENDFOR
			\STATE $\vect{x}^t_{t} \gets \vect{x}^t_{t,K+1}$, play $\vect{x}^i_{t}$ and incurs loss $f^i_t (\vect{x}^i_t)$
			\STATE Receive $\mathcal{F}^i_t = \curlybracket{s \in [T]: s + d^i_s - 1 = t}$
			\IF {$\mathcal{F}^i_t = \emptyset$}  
				\STATE do nothing 
			\ELSE
				\STATE $\vect{g}^i_{t,1} \gets \sum_{s \in \mathcal{F}^i_t} \nabla f^i_{s}(\vect{x}^i_{s,1})$
				\FOR{$k=1$ to $K$}
					\STATE Exchange $\vect{g}^{i}_{t,k}$ with neighbours $\mathcal{N}(i)$
					\STATE \label{eq:gradient-avg}$\vect{d}^{i}_{t,k} \gets  \sum_{j \in \mathcal{N}(i)} w_{ij} \vect{g}^j_{t,k}$
					\STATE \label{eq:gradient-tracking1} $\vect{g}^i_{t,k+1} \gets \sum_{s \in \mathcal{F}^i_t}\parenthese{ \nabla f^i_{s}(\vect{x}^i_{s,k+1}) -  \nabla f^i_{s}(\vect{x}^i_{s,k})}$ \\
					 \qquad \qquad + $\vect{d}^{i}_{t,k}$
					\STATE Feedback $\langle \vect{d}^i_{t,k}, \cdot \rangle$ 
							to oracles $\mathcal{O}_{i,k}$
				\ENDFOR
			\ENDIF
		\ENDFOR
	\ENDFOR
	\end{algorithmic}
	\caption{De2MFW}
	\label{algo:de2mfw}
	\end{algorithm}

%% file: chapters/experiments.tex
\section{Numerical Experiments}
\label{sec:exp}
We evaluated the performance of our algorithms on the online multiclass logistic regression problem using two datasets: MNIST and FashionMNIST. MNIST is a well-known hand digit dataset containing $60 000$ grayscale images of size $(28 \times 28)$, divided into 10 classes, and FashionMNIST includes images of fashion products with the same configuration. We conducted the experiment using Julia 1.7 on MacOS 13.3 with 16GB of memory. The code is available at \url{https://github.com/tuananhngh/DelayMFW}.

\paragraph{Centralized Setting} Given an iteration $t$, the agent receives a subset $\mathcal{B}_t$ of the form $\vect{b}_t = \curlybracket{\vect{a}_t, y_t} \in \mathbb{R}^m \times \curlybracket{1, \dots, C}$, consisting of the features vector $\vect{a}_t$ and the corresponding label $y_t$. We define the loss function $f_t$ as  
\begin{align}
\label{eq:loss-experiment}
    f_t (\vect{x}) 
    = -\sum_{\vect{b}_t \in \mathcal{B}_t} 
        \sum_{c=1}^C \curlybracket{y^i_t = c} \log \frac{\exp{\scalarproduct{\vect{x}_c, \vect{a}^i_t}}}{\sum_{\ell=1}^C \exp{\scalarproduct{\vect{x}_{\ell}, \vect{a}^i_t}}}
\end{align}
where $\vect{x}$ must satisfy the constraint $\vect{x} \in \mathcal{K}$ such that $\mathcal{K} = \curlybracket{\vect{x} \in \mathbb{R}^{m \times C}, \norm{\vect{x}}_1 \leq r}$. Using the MNIST dataset, we note $m=784$, $C=10$, $r = 8$, $|\mathcal{B}_t| = 60$ and a total of $T=1000$ rounds. To evaluate the performance of the algorithm under different delay regimes, we generated a random sequence of delays $d_t$ such that $d_t \leq d$ for $d \in \curlybracket{21, 41, 61, 81, 101}$. We compared the performance of \ce against DOFW \cite{OFW-delay}, a projection-free algorithm with adversarial delay, and BOLD-MFW \cite{bold-ofw}, an online learning framework designed to handle delayed feedback.
\begin{figure*}[t]
\centering
    \includegraphics{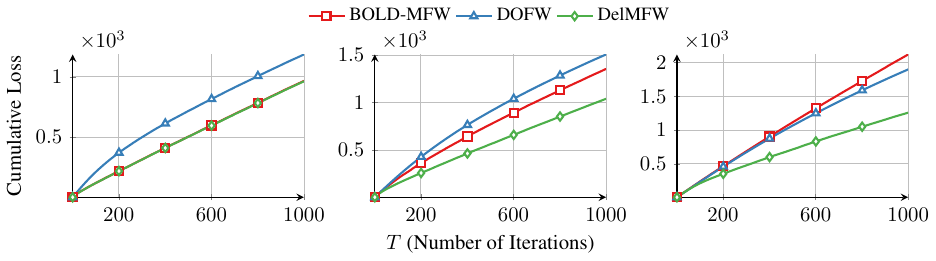}
    \caption{\textit{Cumulative Loss Comparison for Different Delays Regimes. \textit{Left} : Without delay. \textit{Middle} : Maximal delay 21. \textit{Right} : Maximal delay 101}}
    \label{fig:cumloss-centralized}
    \vspace{-20pt}
\end{figure*}
\Cref{fig:cumloss-centralized} displays the performance of the three algorithms under various delay regimes. In the absence of delay, that is, $d=1$ (left figure), \ce and BOLD-MFW have the same performance since both algorithms reduce to MFW \cite{ChenHarshaw18:Projection-Free-Online} with a regret of $\mathcal{O}(\sqrt{T})$. Meanwhile, DOFW is the classical OFW \cite{ofw} that guarantees a regret of $\mathcal{O}(T^{3/4})$. The analysis in \Cref{thm:centralized} suggests that \ce achieves a regret of $\ocal(\sqrt{dT})$ when the delay is upper-bounded by $d$. In the case where $d \leq T^{1/2}$  (middle figure, $d=21$), the dominant term in DOFW is $T^{3/4}$ whereas \ce takes advantage by incurring a regret of order $\sqrt{dT} \leq T^{3/4}$. For $d\geq T^{1/2}$ (right figure, $d=101$), DOFW's regret is dominated by the term $dT^{1/4}$, which is outperformed by \ce, particularly for high values of $d$. This result confirms our theoretical analysis in \Cref{sec:centralized}.

%
%
%
\begin{wrapfigure}{r}{0.5\linewidth}
    \vspace{-15pt}
    \includegraphics[width=0.5\textwidth]{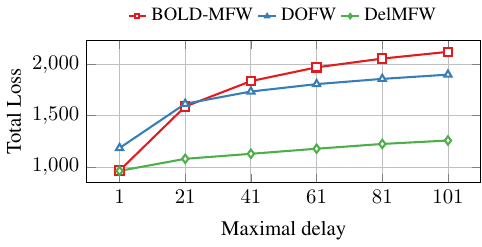}
    \caption{\textit{Total loss of BOLD-MFW, DOFW and \ce when varying delay value.}}
    \label{fig:totalloss-compared}
    \vspace{-20pt}
\end{wrapfigure}
\Cref{fig:totalloss-compared} illustrates the total loss of \ce and the other two algorithms when increasing $d$ to show the sensitivity of each algorithm in the presence of delays. As BOLD is a general framework that can be applied to any base algorithm, it is noticeable that it is susceptible to high levels of delays. This phenomenon has also been observed in \cite{OFW-delay} when utilizing BOLD with OFW, highlighting the need for a customized design algorithms in the context of delayed feedback. 
\paragraph{Distributed Setting}  
In the second experiment, we examined the distributed online multiclass logistic regression problem on the FashionMNIST dataset, using a network of 30 agents. The algorithm was run on four different topologies, including Erdos-Renyi, Complete, Grid, and Cycle. At each iteration $t \in [T]$, each agent $i$ received a subset $\mathcal{B}^i_t$ of the form $\curlybracket{\vect{a}^i_t, y^i_t} \in \mathbb{R}^d \times \curlybracket{1, \dots, C}$, which consisted of the feature vector $\vect{a}^i_t$ and its corresponding label $y^i_t$. The goal was to collaboratively optimize the global loss function $F_t(\vect{x}) = \frac{1}{n} \sum_{i=1}^n f^i_t(\vect{x})$, where the local loss $f^i_t$ was defined in \Cref{eq:loss-experiment}.

For this experiment, we set $m=784$, $C=10$, $r = 32$, $|\mathcal{B}^i_t| = 2$ and $T=1000$ rounds. We are interested in examining the effect of delays on network performance, and thus randomly select $f < n$ agents to have delayed feedback with a maximum value of 501. We compared the total loss on each topology under these conditions, and present the result in \Cref{fig:select-agent}. We observe that the presence of delayed agents has a significant impact on the network performance of Cycle graph as the number of delayed agents increases, while the Complete graph is less affected. This result is consistent with the analysis in \Cref{sec:distributed} because the delay term in the regret bound depends on the connectivity of the communication graph. 

\begin{wrapfigure}{l}{0.5\textwidth}
    \vspace{-15pt}
    \includegraphics[width=0.5\textwidth]{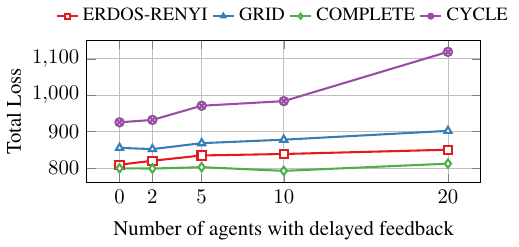}
    \caption{\textit{Total Loss with varying numbers of agents experiencing delayed feedback in the network. $(f=0)$ for no delayed-agents.}}
    \vspace{-30pt}
    \label{fig:select-agent}
\end{wrapfigure}

In \Cref{tab:data_table}, we report the change in total loss when increasing the number of delayed agents. We observe that the average percentage change is smaller for Grid than for Erdos-Renyi when compared with the network of non-delayed agents ($f=0$). This result indicate that the generated Erdos-Renyi graph is more sensitive to the presence of delayed agents. 

\begin{table}[t]
\caption{\textit{Total Loss of the algorithm running on 4 different topology. We randomly select $f < n$ agents to have delay with maximal value to be 501. In parenthesis, the percentage of total loss compared that of no delayed agents in the network (i.e $f=0$).}}
\label{tab:data_table}
\begin{center}
\begin{tabular}{c|c|c|c|c}
\toprule
\diagbox[font=\footnotesize\itshape,linewidth=1pt, width=8em]{f}{Topology} &\textbf{Erd\H{o}s-Rényi} & \textbf{Grid} & \textbf{Complete} & \textbf{Cycle} \\ 
\midrule
0 & 809.37 & 855.62 & 799.49 & 925.72 \\ 
2 & 820.15 (+1.3\%) & 852.15 (-0.4\%) & 798.79 (-0.08\%) & 932.34 (+0.7\%) \\ 
5 & 834.74 (+3.0\%) & 868.52 (+1.4\%) & 802.59 (+0.3\%) & 971.24 (+4.7\%) \\ 
10 & 838.74 (+3.5\%) & 878.04 (+2.5\%) & 792.45 (-0.8\%) & 983.89 (+6.0\%) \\ 
20 & 850.49 (+4.9\%) & 902.30 (+5.3\%) & 812.21 (+1.5\%) & 1119.24 (+18.9\%) \\ 
\bottomrule
\end{tabular}
\end{center}
\vspace{-10pt}
\end{table}

%% file: chapters/conclusion.tex
\section{Concluding Remarks}

In this paper, we propose two algorithms for solving the online convex optimization problem with adversarial delayed feedback in both centralized and decentralized settings. These algorithms achieve optimal $\mathcal{O}(\sqrt{dT})$ regret bounds, where $d$ is the upper bound of the delays. The experimental results show that our algorithms outperform existing solutions in both centralized and decentralized settings, which are predictable by our theoretical analysis. Although the algorithms achieve good performance guarantees for the online convex optimization problem with adversarial delays, they currently rely on exact gradients, which may not be feasible for many real-world applications. Therefore, future research could explore the use of stochastic gradients with variance reduction techniques. Additionally, in decentralized settings, communication delays can be practically challenging, and further improvements are needed in this area. Nevertheless, our work demonstrates the potential of using Frank-Wolfe-type algorithms for solving constraint convex optimization problems under adversarial delays, which is beneficial for learning on edge devices.

%% file: chapters/appendix.tex
We start by recalling notations that will be used throughout the proof of \Cref{thm:centralized,thmm:decen}. 
All notation, are described in the centralized setting. For decentralized setting, if not specify otherwise, we add a superscript $i$ to indicate the dependant of the variable to agent $i$. 
Following the notation in the main section, we note $\mathcal{F}_t = \curlybracket{s \in [T], s + d_s -1 = t}$ as the set of feedback arrives at time $t$ (denoted $\mathcal{F}^i_t$ in decentralized setting). The sum of total delay is denoted by $B := \sum_{t=1}^T d_t $ and we call $d$ the upper bound on all delay value when it is clear from the context. 
For simplicity of notation, we denote by $\nabla f_{t,k} = \nabla f_t (\vect{x}_{t,k})$, $\zeta$ the learning rate of the oracle and $\pi_{\mathcal{K}} (\vect{x})$ the projection of $\vect{x}$ into the constraint set $\mathcal{K}$. 
If $t$ is the time the agent plays the decision $\vect{x}_t$, the feedback is revealed at time $\sigma (t) := t + d_t -1$. 

%% file: chapters/mfw-proofs.tex
\section{Analysis of Algorithm \ref{algo:delmfw}}

\subsection{Auxiliary Lemmas}
\setcounter{lemma}{1}
\begin{lemma}
    Let $\hat{\vect{v}}_t$ be the \ftpl prediction defined in \Cref{eq:ftpl-update} and 
    \begin{align*}
        \vect{v}_t = \argmin_{\vect{v} \in \kcal} \curlybracket{\zeta \sum_{\ell = 1}^{t-1} \scalarproduct{\sum_{s \in \fcal_\ell}\vect{g}_s , \vect{v}} + \scalarproduct{\vect{n}, \vect{v}}}
    \end{align*}
    the prediction of \ftpl with delayed feedback.
    For all $t \in [T]$, we have: 
    \begin{align*}
        \norm{\vect{v}_t - \hat{\vect{v}}_t} \leq \zeta DG \sum_{s < t} \indi_{\curlybracket{s + d_s > t}}
    \end{align*}
    \end{lemma}

\begin{proof}
    Recall that $f_t$ is a linear function defined as $f_t = \scalarproduct{\vect{g}_t, \cdot}$. For the ease of analysis, we call $ \hat{\vect{u}}_t = \zeta \sum_{\ell=1}^{t-1} \vect{g}_{\ell} $ 
    and $\vect{u}_t = \zeta \sum_{\ell=1}^{t-1} \sum_{s \in \mathcal{F}_{\ell}} \vect{g}_s$. 
    We define 
        $h_t(\vect{n}) = \argmin_{\vect{v} \in \mathcal{K}} \curlybracket{ 
                \scalarproduct{\vect{n}, \vect{v}}
                }$. 
    By definition of $\hat{\vect{v}}_t$ and $\vect{v}_t$, 
    we have 
    \begin{align}
        \hat{\vect{v}}_t = \E \bracket{h_t (\vect{n} + \hat{\vect{u}}_t)} 
        = \int_{\vect{n}} h_t (\vect{n} + \hat{\vect{u}}_t) p(\vect{n}) \mathrm{d} \vect{n}
        = \int_{\vect{n}} h_t (\vect{n}) p(\vect{n} - \hat{\vect{u}}_t) \mathrm{d} \vect{n}
    \end{align}

    and 
    \begin{align}
        \vect{v}_t = \E \bracket{h_t (\vect{n} + \vect{u})}
        = \int_{\vect{n}} h_t (\vect{n}) p(\vect{n} - \vect{u}) \mathrm{d} \vect{n}
    \end{align}
    From linearity of expectation and Cauchy-Schwarz inequality,
    \begin{align}
        \norm{\vect{v}_t - \hat{\vect{v}}_t} 
        &\leq \int_{\vect{n}} \norm{h_t(\vect{n})} 
                \abs{p(\vect{n} - \vect{u}) - p(\vect{n} - \hat{\vect{u}}_t)} \mathrm{d} \vect{n} \nonumber \\
        &= \int_{\vect{n}} \norm{h_t(\vect{n}) - h_t (\vect{0})} 
                \abs{p(\vect{n} - \vect{u}) - p(\vect{n} - \hat{\vect{u}}_t)} \mathrm{d} \vect{n} \nonumber \\
        &\leq D \int_{\vect{n}} \abs{p(\vect{n} - \vect{u}) - p(\vect{n} - \hat{\vect{u}}_t)} \mathrm{d} \vect{n} \nonumber \\
        &\leq D L \norm{\vect{u} - \hat{\vect{u}}_t} \nonumber \\
        &\leq \zeta D L G \sum_{s < t} \indi_{\curlybracket{s + d_s > t}}
    \end{align}
The first inequality follows from the fact that $h_t(\vect{n})$ and $h_t(\vect{0})$ are in $\mathcal{K}$. The second inequality is due to the stability of the distribution of $\vect{n}$, where $L$ is the Lipschitz constant of the probability distribution function $p$. Since each function is $G$-Lipschitz, the distance between $\vect{u}$ and $\hat{\vect{u}}_t$ is bounded by $G$ multiplied by the number of functions whose feedback is not received at time $t$, which leads to the last inequality. 
\end{proof}

\setcounter{lemma}{4}
\begin{lemma}[\cite{THANG2022334}]
    \label{lmm:recurrence-1}
    For every $t \in [T]$ and $k \in [K]$. Define $h_{t,k} = f_t (\vect{x}_{t,k+1}) - f_t(\opt{x})$ let $A = \max(3, \frac{G}{\beta D})$ and $\eta_k = \min(1, \frac{A}{k})$, it holds that
    \begin{align}
        h_{t,k} = f_t (\vect{x}_{t,k+1}) - f_t (\vect{x}^*) 
            \leq \frac{2\beta A D^2}{k} 
                + \sum_{k'=1}^k \eta_{k'} \parenthese{\prod_{\ell=k'+1}^k (1-\eta_\ell)} \scalarproduct{\nabla f_{t,k'}, \vect{v}_{t,k'} - \vect{x}^*}
    \end{align}
\end{lemma}
\begin{proof} 
    We state the proof of \cite{THANG2022334} for clarity. The proof is based on an induction on $k$. 
    For $k=1$, $\eta_1 =1$, 
    we have $h_{t,1} = f_t (\vect{x}_{t,2}) - f_t (\vect{x}^*) \leq GD$
    since $f_t$ is $G$-Lipschitz and the constraint set $\kcal$ is bounded by $D$ (\Cref{assum:boundedness,assum:lipschitz}). More over, we have 
    $ 2\beta A D^2 + \scalarproduct{\nabla f_{t,1}, \vect{x}_{t,1} - \vect{x}^*} \geq 2 \beta A D^2 - \scalarproduct{\nabla f_{t,1}, \vect{x}_{t,1} - \vect{x}^*} 
        \geq 2 \beta A D^2 - GD \geq GD$
    by assuming $A \geq \frac{G}{\beta D}$. We have then $h_{t,1} \leq 2\beta A D^2 + \scalarproduct{\nabla f_{t,1}, \vect{x}_{t,1} - \vect{x}^*}$
    Assume that the inequality holds for $k-1$, we now prove for $k$.
    By definition of $h_{t,k}$, we have 
    \begin{align}
        &h_{t,k} 
        \leq (1 - \eta_k) h_{t,k-1} + \eta_k \scalarproduct{\nabla f_{t,k}, \vect{v}_{t,k}-\vect{x}^*} + \eta_k^2 \frac{\beta D^2}{2} \nonumber \\
        & \leq \parenthese{1 - \eta_k} \bracket{
            \frac{2\beta A D^2}{k-1} 
                + \sum_{k'=1}^{k-1} \eta_{k'} \parenthese{\prod_{\ell=k'+1}^{k-1} (1-\eta_\ell)} \scalarproduct{\nabla f_{t,k'}, \vect{v}_{t,k'} - \vect{x}^*}
        } \nonumber \\
            & \quad + \eta_k \scalarproduct{\nabla f_{t,k}, \vect{v}_{t,k}-\vect{x}^*} + \eta_k^2 \frac{\beta D^2}{2} \nonumber \\
        & \leq \parenthese{1 - \eta_k} \bracket{
            \frac{2\beta A D^2}{k-1} 
                + \sum_{k'=1}^{k-1} \eta_{k'} \bracket{\prod_{\ell=k'+1}^{k-1} (1-\eta_\ell)} \scalarproduct{\nabla f_{t,k'}, \vect{v}_{t,k'} - \vect{x}^*}
        } \nonumber \\
            & \quad + \eta_k \prod_{\ell = k+1}^k \parenthese{1 - \eta_\ell} \scalarproduct{\nabla f_{t,k}, \vect{v}_{t,k}-\vect{x}^*} + \eta_k^2 \frac{\beta D^2}{2} \nonumber \\
        & \leq \parenthese{1 - \eta_k} \frac{2\beta A D^2}{k-1} 
            + \sum_{k' = 1}^{k-1} \eta_{k'} \bracket{\prod_{\ell = k'+1}^k (1 - \eta_\ell)} \scalarproduct{\nabla f_{t,k'}, \vect{v}_{t,k'} - \vect{x}^*} \nonumber \\
            & \quad + \eta_k \prod_{\ell = k+1}^k \parenthese{1 - \eta_{\ell}} \scalarproduct{\nabla f_{t,k}, \vect{v}_{t,k}-\vect{x}^*} + \eta_k^2 \frac{\beta D^2}{2} \nonumber \\
        & \leq \parenthese{1 - \eta_k} \frac{2 \beta A D^2}{k-1}
            + \sum_{k'=1}^k \eta_{k'} \bracket{\prod_{\ell = k' + 1}^k \parenthese{1 - \eta_{\ell}}} \scalarproduct{\nabla f_{t,k'}, \vect{v}_{t,k'} - \vect{x}^*} 
                + \eta_k^2 \frac{\beta D^2}{2} \nonumber \\
        & \leq \frac{2 \beta A D^2}{k-1} - \frac{2 \beta A^2 D^2}{k \parenthese{k-1}} + \frac{\beta A^2 D^2}{2k^2} 
            + \sum_{k'=1}^k \eta_{k'} \bracket{\prod_{\ell = k' + 1}^k \parenthese{1 - \eta_{\ell}}} \scalarproduct{\nabla f_{t,k'}, \vect{v}_{t,k'} - \vect{x}^*} \nonumber \\
        & \leq \frac{2 \beta A D^2}{k-1} - \frac{2 \beta A^2 D^2}{k(k-1)} + \frac{\beta A^2 D^2}{2k(k-1)} 
            + \sum_{k'=1}^k \eta_{k'} \bracket{\prod_{\ell = k' + 1}^k \parenthese{1 - \eta_{\ell}}} \scalarproduct{\nabla f_{t,k'}, \vect{v}_{t,k'} - \vect{x}^*} \nonumber \\
        & \leq \frac{2 \beta A D^2}{k-1} - \frac{2 \beta A^2 D^2 - \frac{\beta}{2} A^2 D^2}{k(k-1)}
            + \sum_{k'=1}^k \eta_{k'} \bracket{\prod_{\ell = k' + 1}^k \parenthese{1 - \eta_{\ell}}} \scalarproduct{\nabla f_{t,k'}, \vect{v}_{t,k'} - \vect{x}^*} \nonumber \\
        & \leq \frac{2 \beta A D^2}{k-1} - \frac{\beta A^2 D^2}{k(k-1)} 
            + \sum_{k'=1}^k \eta_{k'} \bracket{\prod_{\ell = k' + 1}^k \parenthese{1 - \eta_{\ell}}} \scalarproduct{\nabla f_{t,k'}, \vect{v}_{t,k'} - \vect{x}^*} \nonumber \\
        & \leq \frac{2 \beta AD^2}{k} 
            + \sum_{k'=1}^k \eta_{k'} \bracket{\prod_{\ell = k' + 1}^k \parenthese{1 - \eta_{\ell}}} \scalarproduct{\nabla f_{t,k'}, \vect{v}_{t,k'} - \vect{x}^*} \nonumber
    \end{align}
    where the last inequality follows from the fact that $\beta A^2 \geq 2 \beta A$ for $A \geq 3$ and $\frac{1}{k-1} - \frac{1}{k(k-1)} \leq \frac{1}{k}$.
\end{proof}

\subsection{Proof of  Theorem \ref{thm:centralized}}
\setcounter{theorem}{0}
\begin{theorem}
    \label[theorem]{thm:centralized1}
    Given a constraint set $\mathcal{K}$. Let $A = \max\curlybracket{3, \frac{G}{\beta D}}$, $\eta_k = \min \curlybracket{1,\frac{A}{k}}$, and $K = \sqrt{T}$. Suppose that \Cref{assum:boundedness,assum:lipschitz,assum:smooth} hold true. If we choose FTPL as the underlying oracle and set $\zeta = \frac{1}{G \sqrt{B}}$, the regret of \Cref{algo:delmfw} is 
    \begin{align}
        \sum_{t=1}^T &\bracket{f_t(\vect{x}_t) - f_t(\vect{x}^*)} 
        \leq 2 \beta A D^2 \sqrt{T}
        + 3(A+1)\parenthese{DG\sqrt{B} + \rcal_{T, \ocal}}
    \end{align}
    where $B = \sum_{t=1}^T d_t$, the sum of all delay values and $\rcal_{T,\ocal}$ is the regret of FTPL with respect to the current choice of $\zeta$
    \end{theorem}
\begin{proof}
Let $\vect{x}_{t,1}, \ldots, \vect{x}_{t,K+1}$ be the sequence of sub-iterate for a fixed time step $t$. Using Frank-Wolfe updates and smoothness of $f_t$, we have
\begin{align}
\label{eq:smooth}
    f_t(\vect{x}_{t,k+1})& - f_t(\vect{x}^*) 
     = f_t (\vect{x}_{t,k} + \eta_k (\vect{v}_{t,k} - \vect{x}_{t,k})) - f_t(\vect{x}^*) \nonumber \\
    & \leq f_t (\vect{x}_{t,k}) - f_t(\vect{x}^*) 
        + \eta_k \scalarproduct{\nabla f_{t,k}, \vect{v}_{t,k} - \vect{x}_{t,k}}
        + \eta_k^2 \frac{\beta}{2} \norm{\vect{v}_{t,k} - \vect{x}_{t,k}}^2 \nonumber \\
    & \leq f_t(\vect{x}_{t,k} ) - f_t (\vect{x}^*) 
         + \eta_k \scalarproduct{\nabla f_{t,k}, \vect{v}_{t,k}-\vect{x}_{t,k}} + \eta_k^2 \frac{\beta D^2}{2} \nonumber \quad \text{($\kcal$ is bounded)} \\
    & \leq f_t(\vect{x}_{t,k} ) - f_t (\vect{x}^*) 
         + \eta_k \bracket{
                \scalarproduct{\nabla f_{t,k}, \vect{v}_{t,k}-\vect{x}^*}
                + \scalarproduct{\nabla f_{t,k}, \vect{x}^*-\vect{x}_{t,k}} 
            }
         + \eta_k^2 \frac{\beta D^2}{2} \nonumber \\
    & \leq f_t(\vect{x}_{t,k} ) - f_t (\vect{x}^*) 
        + \eta_k \bracket{
                \scalarproduct{\nabla f_{t,k}, \vect{v}_{t,k}-\vect{x}^*}
                + f_t(\vect{x}^*) - f_t(\vect{x}_{t,k})
            }  
        + \eta_k^2 \frac{\beta D^2}{2} \nonumber \quad \text{(by convexity of $f_t$)} \\
    & \leq (1 - \eta_k) \bracket{f_t(\vect{x}_{t,k}) - f_t (\vect{x}^*)}
        + \eta_k \scalarproduct{\nabla f_{t,k}, \vect{v}_{t,k}-\vect{x}^*}
        + \eta_k^2 \frac{\beta D^2}{2} \nonumber \\
\end{align}
Let $h_{t,k} = f_t(\vect{x}_{t,k+1}) - f_t(\vect{x}^*)$, \cref{eq:smooth} becomes
\begin{align}
\label{eq:bound-h}
    h_{t,k} \leq (1 - \eta_k) h_{t,k-1} + \eta_k \scalarproduct{\nabla f_{t,k}, \vect{v}_{t,k}-\vect{x}^*} + \eta_k^2 \frac{\beta D^2}{2}
\end{align}
A direct application of \Cref{lmm:recurrence-1} for $k=K$ yields
\begin{align}
\label{eq:bound-h1}
    f_t (\vect{x}_{t,K+1}) - f_t (\vect{x}^*) 
        \leq \frac{2\beta A D^2}{K} 
            + \sum_{k'=1}^K \eta_{k'} \bracket{\prod_{\ell=k+1}^K (1-\eta_\ell)} \scalarproduct{\nabla f_{t,k}, \vect{v}_{t,k} - \vect{x}^*}
\end{align}
Following the notation from $\Cref{algo:delmfw}$ and \Cref{lmm:ftpl-cost-delay}. For a fixed time $t$ and any sub-iterate $k$, $\vect{v}_{t,k}$ and $\hat{\vect{v}}_{t,k}$ are respectively the predictions of the oracle $\ocal_k$ under delayed and non-delayed feedback, the scalar product of \cref{eq:bound-h1} over $T$-round is written as  
\begin{align}
\label[equation]{eq:scalarv0}
    \sum_{t=1}^T \scalarproduct{\nabla f_{t,k}, \vect{v}_{t,k}-\vect{x}^*} 
    &= \sum_{t=1}^T\scalarproduct{\nabla f_{t,k}, \vect{v}_{t,k} - \hat{\vect{v}}_{t,k}}
        + \sum_{t=1}^T\scalarproduct{\nabla f_{t,k}, \hat{\vect{v}}_{t,k} - \vect{x}^*} \nonumber \\
    &\leq  \sum_{t=1}^T\scalarproduct{\nabla f_{t,k}, \vect{v}_{t,k} - \hat{\vect{v}}_{t,k}}
        + \sum_{t=1}^T\scalarproduct{\nabla f_{t,k}, \hat{\vect{v}}_{t,k} - \vect{x}^*}
\end{align}
In the first term on the right hand side of \Cref{eq:scalarv0}, using the Cauchy-Schwartz inequality and \Cref{lmm:ftpl-cost-delay}, we have
\begin{align}
\label[equation]{eq:scalarv1}
    \sum_{t=1}^T \scalarproduct{\nabla f_{t,k}, \vect{v}_{t,k} - \hat{\vect{v}}_{t,k}}
    \leq \sum_{t=1}^T \norm{\nabla f_{t,k}} \norm{\vect{v}_{t,k} - \hat{\vect{v}}_{t,k}}
    \leq \zeta DG^2 \sum_{t=1}^T \sum_{s < t} \indi_{s + d_s > t} 
    \leq \zeta DG^2 B
\end{align}
Recall that the objective function of the oracle $\ocal_k$ in non-delay setting is $\scalarproduct{\nabla f_{t,k}, \cdot}$ and $\hat{\vect{v}}_{t,k}$ is its prediction at time $t$, the second term of \Cref{eq:scalarv0} is bounded by the regret of $\ocal_k$ in non-delay setting, $\rcal_{T, \ocal}$. Specifically, we have
\begin{align}
\label[equation]{eq:scalarv2}
    \sum_{t=1}^T \scalarproduct{\nabla f_{t,k}, \hat{\vect{v}}_{t,k}} 
    \leq \min_{\vect{x} \in \kcal} \sum_{t=1}^T \scalarproduct{\nabla f_{t,k}, \vect{x}} + \rcal_{T, \ocal} 
    \leq \sum_{t=1}^T \scalarproduct{\nabla f_{t,k}, \vect{x}^*} + \rcal_{T, \ocal}
\end{align}
Recall that $\vect{x}_t = \vect{x}_{t,K+1}$, combining \Cref{eq:scalarv0,eq:scalarv1,eq:scalarv2} and summing \Cref{eq:bound-h1} over $T$-rounds yields
\begin{align}
    \sum_{t=1}^T \bracket{f_t (\vect{x}_t) - f_t (\vect{x}^*) }
    \leq \frac{2\beta A D^2}{K} T 
        + \sum_{k'=1}^K \eta_{k'} \bracket{\prod_{\ell=k+1}^K (1-\eta_\ell)} \bracket{
            \zeta DG^2 B + \rcal_{T, \ocal}
        }
\end{align}

Let $\eta_k = \frac{A}{k}$, we have
\begin{align}
\label{eq:bound-prod-k}
    \prod_{k'=k+1}^K (1-\eta_{k'}) 
    \leq e^{-\sum_{k'=k+1}^K \eta_{k'}} 
    \leq e^{-\sum_{k'=\ell}^K \frac{A}{k'}}
    \leq e^{-A\int_{k+2}^K ds/s}
    \leq \parenthese{\frac{k+2}{K}}^A
\end{align}
We have then,
\begin{align}
\label{eq:bound-prod-k2}
    \sum_{k = 1}^K \eta_k \bracket{\prod_{k'=k+1}^K (1-\eta_{k'})} 
    &\leq \min \curlybracket{1, \frac{A}{K}} + \min \curlybracket{1, \frac{A}{K-1}} + \min \curlybracket{1, \frac{A}{K-2}} 
        + \sum_{k=1}^{K-3} \frac{A}{k} \bracket{\frac{k+2}{K}}^A \nonumber \\
    & \leq 3 \min \curlybracket{1, \frac{A}{K-2}} + \frac{A}{K}\sum_{k=1}^{K-3} \frac{k+2}{k} \bracket{\frac{k+2}{K}}^{A-1} \nonumber \\
    & \leq 3 + \frac{3A}{K} \sum_{k=1}^{K-3} \bracket{\frac{k+2}{K}}^{A-1} \nonumber \\
    & \leq 3(A+1)
\end{align}
From \Cref{eq:bound-prod-k2}, we deduce that
\begin{align}
    \sum_{t=1}^T \bracket{f_t (\vect{x}_t) - f_t (\vect{x}^*) }
    &\leq \frac{2\beta A D^2}{K} T 
        + \sum_{k'=1}^K \eta_{k'} \parenthese{\prod_{\ell=k+1}^K (1-\eta_\ell)} \parenthese{
            \zeta DG^2 B + \rcal_{T, \ocal}
        } \nonumber \\
    & \leq \frac{2\beta A D^2}{K} T + 3(A+1) \parenthese{\zeta DG^2 B + \rcal_{T, \ocal}} \nonumber \\
\end{align}
The theorem follows by letting $\zeta = \frac{1}{G\sqrt{B}}$, $K=\sqrt{T}$ and choosing the oracle as FTPL with regret $\rcal_{T, \ocal}$.
\end{proof}

%% file: chapters/decentralized-proofs2.tex
\section{Analysis of Algorithm \ref{algo:de2mfw}}
\label{chap:decentralized2}
\newcommand{\avg}[1]{\ensuremath{\frac{1}{n} \sum_{i=1}^{n} \bm{#1}^{i}_{t,k}}}
\newcommand{\bbar}[2]{\ensuremath{\overline{\bm{#1}}_{t,#2}}}
\newcommand{\avgn}{\ensuremath{\frac{1}{n} \sum_{i=1}^{n}}}
\newcommand{\ilocal}[1]{\ensuremath{\bm{#1}^{i}_{t,k}}}
\newcommand{\ilocalplus}[2]{\ensuremath{\bm{#1}^{i}_{t,#2}}}
\newcommand{\cat}[2]{\ensuremath{#1^{cat}_{t,#2}}}
\newcommand{\catplus}[3]{\ensuremath{#1^{cat}_{#2,#3}}}
\newcommand{\eigen}{\ensuremath{\lambda (\mathbf{W})}}
\newcommand{\eigenplus}[1]{\ensuremath{\lambda^{#1} (\mathbf{W})}}

\subsection{Additional Notations (only for the analysis)}
In the analysis, for any vector $\vect{x}^i \in \mathbb{R}^d, \forall i \in [n]$, we note $\vect{x}^{cat} \in \mathbb{R}^{dn}$ a column vector defined as $\vect{x}^{cat} := \bracket {\vect{x}^{1 \top}, \ldots, \vect{x}^{n \top}}^{\top}$. 
Let $\overline{\vect{x}}$ be the average of $\vect{x}^i$ over $i \in [n]$, the vector $\overline{\vect{x}}^{cat}$ is a $dn$-vector where we stack $n$-times $\overline{\vect{x}}$ i.e $\overline{\vect{x}}^{cat} := \bracket{\overline{\vect{x}}^{\top}, \ldots, \overline{\vect{x}}^{\top}}$.
For simplicity of notation, we note $\nabla f^i_t (\vect{x}^i_{t,k}) := \nabla f^i_{t,k}$ and $\nabla F_{t,k} := \frac{1}{n} \sum_{i=1}^n \nabla f^i_{t,k} $.
In order to incorporate the delay of agents at each time-step $t$, we define $\nabla f^{cat}_{t,k} $ as described above using the sum of agent's delay feedback, we note then
\begin{align}
\label{def:cat_vec}
    \quad \nabla f^{cat}_{t,k} = \bracket{\sum_{s \in \mathcal{F}^1_t}\nabla f^{1 \top}_{s,k}, \dots, \sum_{s \in \mathcal{F}^n_t}\nabla f^{n \top}_{s,k}}^{\top}
\end{align}
and its homologous in the non-delay setting by $\nabla \hat{f}^{cat}_{t,k} = \bracket{\nabla f^{1 \top}_{t,k}, \dots, \nabla f^{n \top}_{t,k}}^{\top}$

The variables $\cat{\vect{d}}{k}, \cat{\hat{\vect{d}}}{k}$ and $\cat{\vect{g}}{k}, \cat{\hat{\vect{g}}}{k} $ are defined similarly as described. Lastly, we define the slack variable $\ilocal{\delta} := \nabla f^i_{t,k} - \nabla f^i_{t,k-1}$, then the definition of $\overline{\vect{\delta}}_{t,k}, \cat{\vect{\delta}}{k} $ and $\overline{\vect{\delta}}^{cat}_{t,k} $ followed. 
\subsection{Auxiliary lemmas}
\begin{claim}
    \label{clm:claim-1}
    In the analysis, we make use of the following bounds 
                \begin{align*}
            \norm{\overline{\vect{x}}_{t,k} - \vect{x}^i_{t,k}} 
            \leq \frac{2C_d}{k}
        \end{align*}
        \begin{align*}
            \norm{\vect{x}^i_{t,k+1} - \vect{x}^i_{t,k}} 
            \leq \frac{4 C_d + AD}{k}
        \end{align*}
    \end{claim}
    \begin{claimproof}
        For the first bound, recall the definition of \fw-update in \Cref{algo:de2mfw} and using \Cref{lmm:decision-distance}, we have
            \begin{align*}
            &\norm{\overline{\vect{x}}_{t,k} - \vect{x}^i_{t,k}}
            = \norm{
                 \parenthese{1-\eta_{k-1}} \parenthese{\overline{\vect{x}}_{t,k-1} - \vect{y}^i_{t,k-1}}
                 + \eta_{k-1} \parenthese{\overline{\vect{v}}_{t,k-1}  - \vect{v}^i_{t,k-1}}
            } \\ 
            &\leq \frac{C_d}{k-1} - \frac{A C_d}{(k-1)^2} + \frac{AD}{k-1} 
            \leq \frac{C_d}{k-1} - \bracket{\frac{A C_d}{(k-1)^2} - \frac{AD}{k-1}} 
            \leq \frac{C_d}{k-1} - \bracket{\frac{A C_d - AD}{(k-1)^2}} 
            \leq \frac{C_d}{k-1} \leq \frac{2C_d}{k}
        \end{align*}
        Applying the first bound on the second one yiels 
        \begin{align*}
            \norm{\vect{x}^i_{t,k+1} - \vect{x}^i_{t,k}} 
            \leq \norm{\ilocalplus{x}{k+1} - \bbar{x}{k+1}} 
                + \norm{\bbar{x}{k+1} - \bbar{x}{k}}
                + \norm{\bbar{x}{k} - \ilocalplus{x}{k}} 
            \leq \frac{2C_d}{k+1} + \frac{AD}{k} + \frac{2C_d}{k}
            \leq \frac{4C_d + AD}{k} \\
        \end{align*}

\end{claimproof}
\setcounter{lemma}{2}
\begin{lemma}
\label{lmm:decision-distance1}
Define $C_d = k_0 \sqrt{n}D$, for all $t \in [T]$, $k \in [K]$, we have
    \begin{align}
    \max_{i \in [1,n]} \norm{\vect{y}^{i}_{t,k} - \overline{\vect{x}}_{t,k}} \leq \frac{C_d}{k}
\end{align}
\end{lemma}
\begin{proof}
\label{sec:proof-decision-distance}
We prove the lemma by induction, we first note that
\begin{align}
    \norm{\vect{y}^{cat}_{t,k} - \overline{\vect{x} }^{cat}_{t,k}}
        &= \norm{
            \parenthese{\mathbf{W} \otimes I_d}\vect{x}^{cat}_{t,k} - \parenthese{\frac{1}{n} \mathbf{1}_n \mathbf{1}_n^T}\vect{x}^{cat}_{t,k}
            } \nonumber \\
        &= \norm{
            \bracket{\parenthese{\mathbf{W} - \frac{1}{n}\mathbf{1}_n \mathbf{1}_n^T}\otimes I_d}\vect{x}^{cat}_{t,k}
        } \nonumber \\
        &= \norm{
            \bracket{\parenthese{\mathbf{W} - \frac{1}{n}\mathbf{1}_n \mathbf{1}_n^T}\otimes I_d}\parenthese{\vect{x}^{cat}_{t,k} - \overline{\vect{x} }^{cat}_{t,k}}
        } \nonumber \\
        & \leq \eigen \norm{\vect{x}^{cat}_{t,k} - \overline{\vect{x} }^{cat}_{t,k}}
\end{align}
Let $C_d = k_0 \sqrt{n}D$, the base case is verified for $k \in [1,k_0]$ since $\norm{\cat{\vect{x}}{k} - \overline{\vect{x} }^{cat}_{t,k}} \leq \sqrt{n}D \leq \frac{C_d}{k}$. Suppose that the hypothesis is verified for $k-1 \geq k_0$, we have
\begin{align}
    \norm{\vect{y}^{cat}_{t,k} - \overline{\vect{x} }^{cat}_{t,k}}
        &\leq \eigen \norm{
                \vect{x}^{cat}_{t,k} - \overline{\vect{x} }^{cat}_{t,k}
                } \nonumber \\
        &=  \eigen \norm{
                \parenthese{1-\eta_{k-1}} \parenthese{
                    \vect{y}^{cat}_{t,k-1} - \overline{\vect{x}}^{cat}_{t,k-1}
                } 
                + \eta_{k-1} \parenthese{
                    \vect{v}^{cat}_{t,k-1} - \overline{\vect{v}}^{cat}_{t,k-1}
                    }
                } \nonumber\\
        &\leq \eigen \norm{
                    \vect{y}^{cat}_{t,k-1} - \overline{\vect{x}}^{cat}_{t,k-1}
                    }
            + \eigen \frac{\sqrt{n}D}{k-1} \nonumber \\
        &\leq \eigen \parenthese{
            \frac{C_d + \sqrt{n}D}{k-1}  
        } \nonumber \\
        &\leq \eigen C_d \frac{k_0 + 1}{k_0 (k-1)} \nonumber \\
        &\leq \frac{C_d}{k}
\end{align}
where we use the induction hypothesis in the third inequality and the last inequality follows the fact that $\eigen \frac{k_0+1}{k_0 (k-1)} \leq \frac{k-1}{k} \cdot \frac{1}{k-1} \leq \frac{1}{k}$. We conclude the proof by noting that
\begin{align*}
    \max_{i \in [1,n]} \norm{\vect{y}^{i}_{t,k} - \overline{\vect{x}}_{t,k}} 
    \leq \sqrt{\sum_{i=1}^n \norm{\vect{y}^{i}_{t,k} - \overline{\vect{x}}_{t,k}}^2}
    = \norm{\vect{y}^{cat}_{t,k} - \overline{\vect{x}}^{cat}_{t,k}}
    \leq \frac{C_d}{k}
\end{align*}
\end{proof}
\setcounter{lemma}{5}
\begin{lemma}
\label{lmm:gradient-distance}
Define $C_g = \sqrt{n}\max \curlybracket{\eigen \parenthese{G + \frac{\beta D}{1-\eigen}}, k_0\beta \parenthese{4C_d + AD}}$ and recall the definition of $\nabla F_{t,k} := \avgn \nabla f^i_{t,k}$. For all $t \in [T]$, $k \in [K]$, we have
    \begin{align}
    \max_{i \in [1,n]} \norm{\vect{d}^{i}_{t,k} - \nabla F_{t,k}} 
    \leq \frac{C_g}{k}
\end{align}
\end{lemma}

\begin{proof}
\label{sec:proof-gradient-distance}
    We prove the lemma by induction. Following the idea from \cite{xie19b}, we have 
        \begin{align}
        \label{eq:gradient_consensus}
            \norm{\cat{\hat{\vect{d}}}{k} - \nabla \cat{F}{k}}
            &= \norm{
                \parenthese{\mathbf{W} \otimes I_d} \cat{\hat{\vect{g}}}{k} - \parenthese{\frac{1}{n} \mathbf{1}_n \mathbf{1}_n^T} \cat{\hat{\vect{g}}}{k}
                } \nonumber \\
            &= \norm{
                \bracket{\parenthese{\mathbf{W} - \frac{1}{n}\mathbf{1}_n \mathbf{1}_n^T}\otimes I_d}\cat{\hat{\vect{g}}}{k}
            } \nonumber \\
            &= \norm{
                \bracket{\parenthese{\mathbf{W} - \frac{1}{n}\mathbf{1}_n \mathbf{1}_n^T}\otimes I_d}\parenthese{\cat{\hat{\vect{g}}}{k} - \nabla \cat{F}{k}}
            } \nonumber \\
            & \leq \eigen \norm{\cat{\hat{\vect{g}}}{k} - \nabla \cat{F}{k}}
        \end{align}
    where the third equality and the last inequality are verified since $\mathbf{W} \cdot \nabla \cat{F}{k} = \frac{1}{n} \mathbf{1}_n \mathbf{1}_n^T \cdot \nabla\cat{F}{k} = \nabla \cat{F}{k}$ and $\norm{ \mathbf{W} -  \frac{1}{n} \mathbf{1}_n \mathbf{1}_n^T} \leq \eigen$ \cite[Lemma 16]{chocosgd}. 
    Using the gradient tracking step and \Cref{eq:gradient_consensus}, we have
        \begin{align}
        \label{eq:recurrence_d_F}
            \norm{\cat{\hat{\vect{d}}}{k} - \nabla \cat{F}{k}}
            &\leq \eigen \norm{
                    \cat{\hat{\vect{g}}}{k}  - \nabla \cat{F}{k}
                    } \nonumber \\
            &=  \eigen \norm{
                    \cat{\vect{\delta}}{k} + \cat{\hat{\vect{d}}}{k-1}
                    - \nabla \cat{F}{k} + \nabla \cat{F}{k-1} - \nabla \cat{F}{k-1}
                    } \nonumber \\
            & \leq \eigen \parenthese{
                \norm{\cat{\hat{\vect{d}}}{k-1} - \nabla \cat{F}{k-1}}
                + \norm{\cat{\vect{\delta}}{k} - \cat{\overline{\vect{\delta}}}{k}}}
             \nonumber \\
            & \leq \eigen \parenthese{
                 \norm{\cat{\hat{\vect{d}}}{k-1} - \nabla \cat{F}{k-1}}
                + \norm{\cat{\vect{\delta}}{k}}}
        \end{align}
    where the last inequality holds since
    \begin{align}
        \norm{\cat{\vect{\delta}}{k} - \cat{\overline{\vect{\delta}}}{k}}^2  
        = \sum_{i=1}^n \norm{\ilocal{\delta} - \bbar{\delta}{k}}^2  
        = \sum_{i=1}^n \norm{\ilocal{\delta}}^2 - n \norm{\bbar{\delta}{k}}^2
        \leq \sum_{i=1}^n \norm{\ilocal{\delta}}^2
        = \norm{\cat{\vect{\delta}}{k}}^2
    \end{align}
    Moreover, using the smoothness of $f_t$ and \Cref{clm:claim-1}, we have
    \begin{align}
    \label{eq:bound_slack_1}
        \norm{\cat{\vect{\delta}}{k}}^2 = \sum_{i=1}^n \norm{\ilocal{\delta}}^2
        &\leq \sum_{i=1}^n \norm{\nabla f^i_{t,k} - \nabla f^i_{t,k-1}}^2 
            \leq \sum_{i=1}^n \beta^2 \norm{\vect{x}^i_{t,k}- \vect{x}^i_{t,k-1}}^2 \nonumber \\
        &\leq n \beta^2 \parenthese{\frac{4 C_d + AD}{k-1}}^2
    \end{align}
    Thus, we have $\norm{\cat{\vect{\delta}}{k}} \leq \sqrt{n} \beta \frac{4C_d + AD}{k-1}$. For the base case  $k=1$, we have
    \begin{align}
    \label{eq:bound_d_F_init}
        \norm{\cat{\hat{\vect{d}}}{1} - \nabla F^{cat}_{t,1}}^2
        = \norm{
                \bracket{\parenthese{\mathbf{W} - \frac{1}{n}\mathbf{1}_n \mathbf{1}_n^T}\otimes I_d} \cat{\hat{\vect{g}}}{1}
            }^2
        \leq \eigenplus{2} \norm{\cat{\hat{\vect{g}}}{1}}^2 
        \leq \eigenplus{2}\sum_{i=1}^n \norm{\nabla f^i_{t,1}}^2 
        \leq n \eigenplus{2}G^2
    \end{align}
    We have then $\norm{\cat{\hat{\vect{d}}}{1} - \nabla F^{cat}_{t,1}} \leq \sqrt{n} \eigen G$.
    For $k \in (1, k_0]$, by \Cref{eq:recurrence_d_F}
    \begin{align*}
        \norm{\cat{\hat{\vect{d}}}{k} - \nabla \cat{F}{k}}
        & \leq \eigen \parenthese{
                 \norm{\cat{\hat{\vect{d}}}{k-1}  - \nabla \cat{F}{k-1}}
                + \norm{\cat{\vect{\delta}}{k}}} \nonumber \\
        & \leq \eigen \parenthese{
            \norm{\cat{\hat{\vect{d}}}{k-1}  - \nabla \cat{F}{k-1}}
           + \sqrt{n \beta^2 D^2}} \nonumber \\ 
        & \leq \eigenplus{k-1}
                \norm{\cat{\hat{\vect{d}}}{1}  - \nabla F^{cat}_{t,1}}
                + \sum_{\tau=1}^k  \eigenplus{\tau} \sqrt{n} \beta D \\
        & \leq \eigenplus{k} \sqrt{n} G 
                + \frac{\eigen}{1 - \eigen}\sqrt{n} \beta D \\
        & \leq \eigen \sqrt{n} \parenthese{G + \frac{\beta D}{1-\eigen}}
    \end{align*}
    where in the second inequality, we use smoothness of $f_t$ and bound the distance $\norm{\vect{x}^i_{t,k} - \vect{x}^i_{t,k-1} }$ by the diameters $D$. The third inequality resulted from applying the previous one recursively for $k \in \curlybracket{1, \ldots, k-1}$. Using Taylor's expansion of $\eigen$ and the bound in \Cref{eq:bound_d_F_init}, we obtain the fourth inequality.
    Let $C_{g} = \sqrt{n}\max \curlybracket{\eigen \parenthese{G + \frac{\beta D}{1-\eigen}}, k_0\beta \parenthese{4C_d + AD}}$. We claim that $\norm{\cat{\hat{\vect{d}}}{k-1}  - \nabla \cat{F}{k-1}} \leq \frac{C_g}{k-1}$ for all $k-1 \geq k_0$. We prove the claim for round $k$. Using \Cref{eq:recurrence_d_F}, we have
    \begin{align}
        \label{eq:recurrence_d_F1}
            \norm{\cat{\hat{\vect{d}}}{k}  - \nabla \cat{F}{k}}
            & \leq \eigen \parenthese{
                 \norm{\cat{\hat{\vect{d}}}{k-1}  - \nabla \cat{F}{k-1}}
                + \norm{\cat{\vect{\delta}}{k}}} \nonumber \\
            & \leq \eigen \parenthese{
                \frac{C_g}{k-1} + \sqrt{n}\beta  \frac{4C_d + AD}{k-1}
            } \nonumber \\
            & \leq \eigen \parenthese{
                \frac{C_g + \sqrt{n} \beta  \parenthese{4C_d + AD}}{k-1}
            } \nonumber \\
            & \leq \eigen \parenthese{
                C_g \frac{k_0 + 1}{k_0 (k-1)}
            } \nonumber \\
            & \leq \frac{C_g}{k}
    \end{align}
    where the second inequality followed by the induction hypothesis and \Cref{eq:bound_slack_1}. The fourth inequality is a consequence of the definition of $C_g$ and the final inequality resulted from the fact that $\eigen \frac{k_0 + 1}{k_0 (k-1)} \leq \frac{1}{k}$ as $k > k_0$. We conclude the proof by noting that
    \begin{align*}
        \max_{i \in [1,n]} \norm{\hat{\vect{d}}^{i}_{t,k} - \nabla F_{t,k}}
        \leq \sqrt{\sum_{i=1}^n \norm{\hat{\vect{d}}^{i}_{t,k} - \nabla F_{t,k}}^2}
        = \norm{\cat{\hat{\vect{d}}}{k} - \nabla \cat{F}{k}}
        \leq \frac{C_g}{k}
    \end{align*}
\end{proof}
\setcounter{lemma}{3}
\begin{lemma}
    \label{lmm:ftpl-bound-distributed1}
    For all $t \in [T]$, $k \in [K]$ and $i \in [n]$. Let $\vect{v}^i_{t,k}$ be the output of the oracle $\ocal^i_k$ with delayed feedback and $\hat{\vect{v}}^i_{t,k}$ its homologous in non-delay case. Suppose that \Cref{assum:boundedness,assum:lipschitz} hold true. Choosing \ftpl as the oracle, we have:
        \begin{align}
            \norm{\vect{v}^i_{t,k} - \hat{\vect{v}}^i_{t,k}}
            \leq 2 \zeta \sqrt{n} DG \bracket{\frac{\lambda\parenthese{\mathbf{W}}}{\rho} +1 } \frac{1}{n} \sum_{i=1}^n \sum_{s \leq t}  \indi_{\curlybracket{s + d^i_s > t}}
        \end{align}
    where $\zeta$ is the learning rate, $\lambda(\mathbf{W})$ is the second-largest eigenvalue of $\mathbf{W}$ and $\rho = 1 - \lambda (\mathbf{W})$ is the spectral gap of matrix $\mathbf{W}$.
\end{lemma}
\begin{proof}
\label{sec:proof-ftpl-bound-distributed}
We call $\vect{u}^i_{t} = \zeta \sum_{\ell=1}^t \vect{d}^i_{\ell}$ the accumulated delayed feedback of the oracle of agent $i$ and $\hat{\vect{u}}^i_{t} = \zeta \sum_{\ell=1}^t \hat{\vect{d}}^i_{\ell}$ its homologous in non-delay setting. Using the same computation in the proof of \Cref{lmm:ftpl-cost-delay}, we have 
\begin{align}
    \label{eq:bound-v-vhat-distributed}
    \norm{\ilocal{v} - \ilocal{\hat{v}}}
    \leq D \norm{\ilocal{u} - \ilocal{\hat{u}}} 
    \leq D \norm{\cat{\vect{u}}{k} - \cat{\hat{\vect{u}}}{k}} 
    \leq \zeta D \norm{\sum_{\ell=1}^t \bracket{
            \catplus{\vect{d}}{\ell}{k} - \catplus{\hat{\vect{d}}}{\ell}{k}}
        }
\end{align}
For $t \in [T]$, the concatenation of the gradient local average  is expanded as follow
\begin{align}
    \label{eq:wa-surrogate}
    &\vect{d}^{cat}_{t,k} 
    = \parenthese{\mathbf{W} \otimes I_d} \parenthese{\nabla f^{cat}_{t,k} - \nabla f^{cat}_{t,k-1} +\vect{d}^{cat}_{t,k-1}} \nonumber \\
    =& \parenthese{\mathbf{W} \otimes I_d} \parenthese{\nabla f^{cat}_{t,k} - \nabla f^{cat}_{t,k-1}} 
        + \parenthese{\mathbf{W} \otimes I_d}^2 \parenthese{\nabla f^{cat}_{t,k-1} - \nabla f^{cat}_{t,k-2} +\vect{d}^{cat}_{t,k-2}} \nonumber \\
    =& \sum_{\tau = 1}^{k-1} \bracket{\parenthese{\mathbf{W} \otimes I_d}^{k-\tau} \parenthese{\nabla f^{cat}_{t,\tau+1} - \nabla f^{cat}_{t,\tau}}} 
        + \parenthese{\mathbf{W} \otimes I_d}^{k} \nabla f^{cat}_{t,1}  \\
    =& \sum_{\tau = 1}^{k-1} \bracket{\parenthese{\mathbf{W} \otimes I_d}^{k-\tau} 
    \parenthese{\nabla f^{cat}_{t,\tau+1} - \nabla f^{cat}_{t,\tau}}} 
        + \parenthese{\mathbf{W} \otimes I_d}^{k} \nabla f^{cat}_{t,1} 
        - \sum_{\tau=1}^{k-1} \bracket{\nabla F^{cat}_{t,\tau+1} - \nabla F^{cat}_{t,\tau}} - \nabla F^{cat}_{t,1} + \nabla F^{cat}_{t,k} \nonumber \\
    =& \sum_{\tau = 1}^{k-1} \bracket{\bracket{\parenthese{\mathbf{W}^{k-\tau} - \frac{1}{n}\mathbf{1}_n \mathbf{1}_n^T}\otimes I_d} \parenthese{\nabla f^{cat}_{t,\tau+1} - \nabla f^{cat}_{t,\tau}}} 
        + \bracket{\parenthese{\mathbf{W}^{k} - \frac{1}{n}\mathbf{1}_n \mathbf{1}_n^T}\otimes I_d}
            \nabla f^{cat}_{t,1}
                + \parenthese{\frac{1}{n}\mathbf{1}_n \mathbf{1}_n^T \otimes I_d}\nabla f^{cat}_{t,k} \nonumber
\end{align}
Recall the definition of $\nabla f^{cat}_{t,k} = \bracket{\sum_{s \in \mathcal{F}^1_t}\nabla f^{1 \top}_{s,k}, \dots, \sum_{s \in \mathcal{F}^n_t}\nabla f^{n \top}_{s,k}}^{\top}$ and $\nabla \hat{f}^{cat}_{t,k} = \bracket{\nabla f^{1 \top}_{t,k}, \dots, \nabla f^{n \top}_{t,k}}^{\top}$. We have
\begin{align}
    \sum_{\ell = 1}^t \bracket{
        \nabla \catplus{f}{\ell}{k} - \nabla \catplus{\hat{f}}{\ell}{k}
        }
    = \sum_{s < t} \bracket{
        \nabla f^{1\top}_{s ,k}\indi_{\curlybracket{s + d^1_s > t}},
        \ldots ,
        \nabla f^{n\top}_{s ,k}\indi_{\curlybracket{s + d^n_s > t}},
    }^{\top}
\end{align}
Using \Cref{eq:wa-surrogate}, we bound the RHS of \Cref{eq:bound-v-vhat-distributed} as follow
\begin{align}
    &\norm{
        \sum_{\ell=1}^t \bracket{
            \catplus{\vect{d}}{\ell}{k} - \catplus{\hat{\vect{d}}}{\ell}{k}
        }
    }
    \leq \sum_{\tau = 1}^{k-1} \norm{\parenthese{\mathbf{W}^{k-\tau} - \frac{1}{n}\mathbf{1}_n \mathbf{1}_n^T}\otimes I_d} 
        \norm{\sum_{\ell=1}^t \bracket{\nabla f^{cat}_{\ell,\tau+1} - \catplus{\hat{f}}{\ell}{\tau+1} + \catplus{\hat{f}}{\ell}{\tau} - \nabla f^{cat}_{t,\tau}}} \nonumber \\
        & \quad + \norm{\parenthese{\mathbf{W}^{k} - \frac{1}{n}\mathbf{1}_n \mathbf{1}_n^T}\otimes I_d}
            \norm{\sum_{\ell=1}^t \bracket{\nabla f^{cat}_{\ell,1} - \catplus{\hat{f}}{\ell}{1}}} 
            + \parenthese{\frac{1}{n}\mathbf{1}_n \mathbf{1}_n^T \otimes I_d}
                \norm{ \sum_{\ell=1}^t \bracket{\nabla f^{cat}_{\ell,k} - \catplus{\hat{f}}{\ell}{k}}} \nonumber \\
    & \leq 2 \sum_{\tau=1}^k \eigenplus{k-\tau} \norm{
            \sum_{s \leq t} \bracket{
            \nabla f^{1\top}_{s ,\tau}\indi_{\curlybracket{s + d^1_s > t}},
            \ldots ,
            \nabla f^{n\top}_{s ,\tau}\indi_{\curlybracket{s + d^n_s > t}}
            }^{\top}} \nonumber \\
        & \quad + \eigenplus{k} \norm{
            \sum_{s \leq t} \bracket{
                \nabla f^{1\top}_{s ,k}\indi_{\curlybracket{s + d^1_s > t}},
                \ldots ,
                \nabla f^{n\top}_{s ,k}\indi_{\curlybracket{s + d^n_s > t}}
                }^{\top}
            } 
        + \norm{
            \sum_{s \leq t} \bracket{
                \nabla f^{1\top}_{s ,1}\indi_{\curlybracket{s + d^1_s > t}},
                \ldots ,
                \nabla f^{n\top}_{s ,1}\indi_{\curlybracket{s + d^n_s > t}}
                }^{\top}
            } \nonumber \\
    & \leq 2 \sum_{\tau=1}^k \eigenplus{k-\tau}
            \sum_{s \leq t} \norm{\bracket{
            \nabla f^{1\top}_{s ,\tau}\indi_{\curlybracket{s + d^1_s > t}},
            \ldots ,
            \nabla f^{n\top}_{s ,\tau}\indi_{\curlybracket{s + d^n_s > t}}
            }^{\top}} \nonumber \\
        & \quad + \eigenplus{k} 
            \sum_{s \leq t} \norm{ \bracket{
                \nabla f^{1\top}_{s ,k}\indi_{\curlybracket{s + d^1_s > t}},
                \ldots ,
                \nabla f^{n\top}_{s ,k}\indi_{\curlybracket{s + d^n_s > t}}
                }^{\top}
            } 
        + \sum_{s \leq t} \norm{\bracket{
                \nabla f^{1\top}_{s ,1}\indi_{\curlybracket{s + d^1_s > t}},
                \ldots ,
                \nabla f^{n\top}_{s ,1}\indi_{\curlybracket{s + d^n_s > t}}
                }^{\top}
        } \nonumber \\
    & \leq 2 \sum_{\tau=1}^k \eigenplus{k-\tau} \sum_{s \leq t} \sqrt{
                \sum_{i=1}^n \norm{\nabla f^i_{s,\tau} \indi_{\curlybracket{s + d^i_s > t}}}^2 
                }  
        + \eigenplus{k} \sum_{s \leq t} \sqrt{
                \sum_{i=1}^n \norm{\nabla f^i_{s,k} \indi_{\curlybracket{s + d^i_s > t}}}^2 
        } 
        + \sum_{s \leq t} \sqrt{
            \sum_{i=1}^n \norm{\nabla f^i_{s,1} \indi_{\curlybracket{s + d^i_s > t}}}^2 
            } \nonumber \\
    & \leq 2G\sum_{\tau=1}^k \eigenplus{k-\tau} \sum_{s \leq t} \sqrt{\sum_{i=1}^n \indi_{\curlybracket{s + d^i_s > t}}} 
        + G \bracket{\eigenplus{k} + 1} \sum_{s \leq t} \sqrt{\sum_{i=1}^n \indi_{\curlybracket{s + d^i_s > t}}} \nonumber \\
    & \leq 2G \bracket{\frac{\eigen}{1-\eigen} + 1} \sum_{s \leq t} \sqrt{\sum_{i=1}^n \indi_{\curlybracket{s + d^i_s > t}}} \nonumber \\
    & \leq 2 \sqrt{n} G \bracket{\frac{\eigen}{1-\eigen} + 1} \sum_{s \leq t} \avgn \sum_{i=1}^n \indi_{\curlybracket{s + d^i_s > t}} \nonumber
\end{align}
where the fifth inequality resulted from Lipschitzness of $f_t$ and the value of indicator function. 
\end{proof}
%

\subsection{Proofs of Theorem \ref{thm:decentralized}}
\setcounter{theorem}{1}
\begin{theorem}
    \label{thm:decentralized1}
    Given a constraint set $\mathcal{K}$. Let $A = \max\curlybracket{3,\frac{3G}{2\beta D}, \frac{2\beta C_d + C_g}{\beta D}}$, $\eta_k = \min \curlybracket{1,\frac{A}{k}}$, and $K = \sqrt{T}$. Suppose that \Cref{assum:boundedness,assum:lipschitz,assum:smooth} hold true. If we choose FTPL as the underlying oracle and set $\zeta = \frac{1}{G\sqrt{B}}$, the regret of \Cref{algo:de2mfw} is 
    \begin{align}
        \sum_{t=1}^T & \bracket{F_t (\vect{x}^i_t) - F_t(\opt{x})}
        \leq \parenthese{GC_d + 2\beta A D^2}\sqrt{T} + 3(A+1) \bracket{
            2 \sqrt{n} DG \parenthese{\frac{\eigen}{\rho} + 1} \sqrt{B}
            + \rcal_{T, \ocal} } \nonumber \\
    \end{align}
    where $B = \avgn B_i$ such that $B_i$ is the sum of delay values of agent $i$. $C_d = k_0 \sqrt{n} D$ and $C_g = \sqrt{n} \max \curlybracket{\eigen \parenthese{G + \frac{\beta D}{\rho}}, k_0\beta \parenthese{4C_d + AD}}$ are defined in \Cref{lmm:decision-distance,lmm:gradient-distance} and $\rcal_{T,\ocal}$ is the regret of FTPL with the current choice of $\zeta$
\end{theorem}
\begin{proof}
By smoothness of $F_t$, we have
\begin{align}
\label{eq:smoothness}
    &F_t (\bbar{x}{k+1}) - F_t (\opt{x})
    = F_t \parenthese{\bbar{x}{k} + \eta_k \parenthese{\avg{v} - \bbar{x}{k}}} - F_t (\opt{x}) \nonumber \\
    \leq& F_t(\bbar{x}{k}) - F_t(\opt{x})
        + \eta_k \scalarproduct{\nabla F_t (\bbar{x}{k}), \avg{v} - \bbar{x}{k}} 
        + \eta_k^2 \frac{\beta}{2} \norm{\avg{v} - \bbar{x}{k}}^2 \nonumber \\
    \leq& F_t(\bbar{x}{k}) - F_t(\opt{x})
        + \frac{\eta_k}{n}\scalarproduct{\nabla F_t (\bbar{x}{k}), \ilocal{v} - \bbar{x}{k}} 
        + \eta_k^2 \frac{\beta D^2}{2} \\
    \leq& F_t(\bbar{x}{k}) - F_t(\opt{x})
        + \frac{\eta_k}{n} \sum_{i=1}^{n} \bracket{
            \scalarproduct{\nabla F_t (\bbar{x}{k}), \ilocal{v} - \opt{x}}
            + \scalarproduct{\nabla F_t (\bbar{x}{k}), \opt{x} - \bbar{x}{k}}
        }
        + \eta_k^2 \frac{\beta D^2}{2} \nonumber \\
    \leq& \parenthese{1 - \eta_k} \bracket{F_t (\bbar{x}{k}) - F_t (\opt{x})}
        + \frac{\eta_k}{n} \sum_{i=1}^{n} \bracket{
            \scalarproduct{\nabla F_t (\bbar{x}{k}), \ilocal{v} - \opt{x}}
        }
        + \eta_k^2 \frac{\beta D^2}{2} \nonumber \\
    \leq& \parenthese{1 - \eta_k} \bracket{F_t (\bbar{x}{k}) - F_t (\opt{x})}
        + \frac{\eta_k}{n} \sum_{i=1}^{n} \bracket{
            \scalarproduct{\nabla F_t (\bbar{x}{k}), \ilocal{v} - \ilocal{\hat{v}}}
            + \scalarproduct{\nabla F_t (\bbar{x}{k}), \ilocal{\hat{v}} - \opt{x}}
        }
        + \eta_k^2 \frac{\beta D^2}{2} \nonumber \\
    \leq& \parenthese{1 - \eta_k} \bracket{F_t (\bbar{x}{k}) - F_t (\opt{x})}
        + \frac{\eta_k}{n} \sum_{i=1}^{n} \bracket{
            \scalarproduct{\nabla F_t (\bbar{x}{k}), \ilocal{v} - \ilocal{\hat{v}}}
            + \scalarproduct{\ilocal{\hat{d}}, \ilocal{\hat{v}} - \opt{x}}
        } 
        + \eta_k \frac{2\beta C_d + C_g}{k} D 
        + \eta_k^2 \frac{\beta D^2}{2} \nonumber 
\end{align}
where we used convexity of $F_t$ in the fourth inequality and the last inequality follows by observing that 
\begin{align*}
    &\scalarproduct{\nabla F_t (\bbar{x}{k}), \ilocal{\hat{v}} - \opt{x}}
    = \scalarproduct{\nabla F_t (\bbar{x}{k}) - \nabla F_{t,k}, \ilocal{\hat{v}} - \opt{x}}
        + \scalarproduct{\nabla F_{t,k}, \ilocal{\hat{v}} - \opt{x}} \\
    &= \scalarproduct{\nabla F_t (\bbar{x}{k}) - \nabla F_{t,k}, \ilocal{\hat{v}} - \opt{x}}
        + \scalarproduct{\nabla F_{t,k} - \ilocal{\hat{d}}, \ilocal{\hat{v}} - \opt{x}} 
        + \scalarproduct{\ilocal{\hat{d}}, \ilocal{\hat{v}} - \opt{x}} \\
    &\leq \parenthese{
        \beta \norm{\bbar{\vect{x}}{k} - \ilocal{\vect{x}}}
        + \norm{\nabla F_{t,k} - \ilocal{\hat{d}}}
    } D + \scalarproduct{\ilocal{\hat{d}}, \ilocal{\hat{v}} - \opt{x}} \\
    & \leq \frac{2 \beta C_d + C_g}{k}D + \scalarproduct{\ilocal{\hat{d}}, \ilocal{\hat{v}} - \opt{x}} 
\end{align*}
Let $A = \max \curlybracket{3, \frac{3G}{2\beta D}, \frac{2\beta C_d + Cg}{\beta D}}$ and $\eta_k = \frac{A}{k}$, From \Cref{lmm:recurrence-1} we have 
\begin{align}
\label{eq:recurrence-distributed}
    F_t (\bbar{x}{K+1}) - F_t(\opt{x}) 
    \leq \frac{2\beta A D^2}{K} 
        + \avgn \sum_{k=1}^K \eta_k \bracket{\prod_{\ell = k + 1}^K \parenthese{1 - \eta_{\ell}}}
            \bracket{
                \scalarproduct{\ilocalplus{\hat{d}}{k}, \ilocalplus{\hat{v}}{k} - \opt{x}}
                + \scalarproduct{\nabla F_t (\bbar{\vect{x}}{k}), \ilocalplus{v}{k} - \ilocalplus{\hat{v}}{k}}
            } 
\end{align}
Summing \Cref{eq:recurrence-distributed} over $T$-rounds yields,
\begin{align}
\label{eq:recurrence-distributed-2}
    \sum_{t=1}^T &\bracket{F_t (\overline{\vect{x}}_t) - F_t(\opt{x})}
    \leq \frac{2\beta A D^2 T}{K} 
        + \avgn \sum_{k=1}^K \eta_k \prod_{\ell = k + 1}^K \parenthese{1 - \eta_{\ell}}
            \sum_{t=1}^T \bracket{
                \scalarproduct{\ilocalplus{\hat{d}}{k}, \ilocalplus{\hat{v}}{k} - \opt{x}}
                + \scalarproduct{\nabla F_t (\bbar{\vect{x}}{k}), \ilocalplus{v}{k} - \ilocalplus{\hat{v}}{k}}
            } \nonumber \\
    & \leq \frac{2\beta A D^2 T}{K} 
        + \avgn \sum_{k=1}^K \eta_k \prod_{\ell = k + 1}^K \parenthese{1 - \eta_{\ell}}
            \bracket{
                \rcal_{T, \ocal} 
                + 2 \zeta \sqrt{n} DG^2 \parenthese{\frac{\eigen}{1-\eigen} + 1} \avgn  \sum_{t=1}^T \sum_{s \leq t}  \indi_{\curlybracket{s + d^i_s > t}}
            } \nonumber \\
    & \leq \frac{2\beta A D^2 T}{K} + 3(A+1) \bracket{
                \rcal_{T, \ocal} 
                + 2 \zeta \sqrt{n} DG^2 \parenthese{\frac{\eigen}{1-\eigen} + 1} \avgn \sum_{t=1}^T \sum_{s \leq t}  \indi_{\curlybracket{s + d^i_s > t}}}
\end{align}
From \Cref{eq:recurrence-distributed-2}, we deduce that, for all $i \in [n]$,
\begin{align}
    \sum_{t=1}^T & \bracket{F_t (\vect{x}^i_t) - F_t(\opt{x})}
    \leq \sum_{t=1}^T \bracket{F_t (\vect{x}^i_t) - F_t(\overline{\vect{x}}_t)} 
        +  \sum_{t=1}^T \bracket{F_t (\overline{\vect{x}}_t) - F_t(\opt{x})} \nonumber \\
    & \leq \sum_{t=1}^T G \norm{ \vect{x}^i_t  - \overline{\vect{x}}_t} 
        + \sum_{t=1}^T \bracket{F_t (\overline{\vect{x}}_t) - F_t(\opt{x})} \nonumber \\
    & \leq \frac{GC_d T}{K} + \sum_{t=1}^T \bracket{F_t (\overline{\vect{x}}_t) - F_t(\opt{x})} \nonumber \\
    & \leq \frac{GC_d + 2\beta A D^2}{K} T + 3(A+1) \bracket{
                \rcal_{T, \ocal}  
                + 2 \zeta \sqrt{n} DG^2 \parenthese{\frac{\eigen}{1-\eigen} + 1}  \avgn \sum_{t=1}^T \sum_{s < t} \indi_{\curlybracket{s + d^i_s > t}}} \nonumber \\
    & \leq \frac{GC_d + 2\beta A D^2}{K} T + 3(A+1) \bracket{
                \rcal_{T, \ocal} 
                + 2 \zeta \sqrt{n} DG^2 \parenthese{\frac{\eigen}{1-\eigen} + 1}  \avgn B_i} \nonumber \\
\end{align}
The theorem follows by letting $\zeta = \frac{1}{G \sqrt{B}}$, $K = \sqrt{T}$ and $B = \avgn B_i$. This concludes the proof. 
\end{proof}